\PassOptionsToPackage{unicode}{hyperref}
\PassOptionsToPackage{hyphens}{url}
\PassOptionsToPackage{dvipsnames,svgnames,x11names}{xcolor}
\documentclass[12pt]{article}
\usepackage{setspace}
\usepackage{amsmath,amssymb}
\usepackage{iftex}
\usepackage[normalem]{ulem}
\usepackage{bbm}
\usepackage{hyperref}
\ifPDFTeX
  \usepackage[T1]{fontenc}
  \usepackage[utf8]{inputenc}
  \usepackage{textcomp} 
\else 
  \usepackage{unicode-math}
  \defaultfontfeatures{Scale=MatchLowercase}
  \defaultfontfeatures[\rmfamily]{Ligatures=TeX,Scale=1}
\fi
\usepackage{lmodern}
\ifPDFTeX\else  
\fi
\IfFileExists{upquote.sty}{\usepackage{upquote}}{}
\IfFileExists{microtype.sty}{
  \usepackage[]{microtype}
  \UseMicrotypeSet[protrusion]{basicmath} 
}{}
\makeatletter
\@ifundefined{KOMAClassName}{
  \IfFileExists{parskip.sty}{%
    \usepackage{parskip}
  }{
    \setlength{\parindent}{0pt}
    \setlength{\parskip}{6pt plus 2pt minus 1pt}}
}{
  \KOMAoptions{parskip=half}}
\makeatother
\usepackage{xcolor}
\makeatletter
\ifx\paragraph\undefined\else
  \let\oldparagraph\paragraph
  \renewcommand{\paragraph}{
    \@ifstar
      \xxxParagraphStar
      \xxxParagraphNoStar
  }
  \newcommand{\xxxParagraphStar}[1]{\oldparagraph*{#1}\mbox{}}
  \newcommand{\xxxParagraphNoStar}[1]{\oldparagraph{#1}\mbox{}}
\fi
\ifx\subparagraph\undefined\else
  \let\oldsubparagraph\subparagraph
  \renewcommand{\subparagraph}{
    \@ifstar
      \xxxSubParagraphStar
      \xxxSubParagraphNoStar
  }
  \newcommand{\xxxSubParagraphStar}[1]{\oldsubparagraph*{#1}\mbox{}}
  \newcommand{\xxxSubParagraphNoStar}[1]{\oldsubparagraph{#1}\mbox{}}
\fi
\makeatother

\usepackage{longtable,booktabs,array}
\usepackage{calc} 
\usepackage{etoolbox}
\makeatletter
\patchcmd\longtable{\par}{\if@noskipsec\mbox{}\fi\par}{}{}
\makeatother
\IfFileExists{footnotehyper.sty}{\usepackage{footnotehyper}}{\usepackage{footnote}}
\makesavenoteenv{longtable}
\usepackage{graphicx}
\makeatletter
\def\maxwidth{\ifdim\Gin@nat@width>\linewidth\linewidth\else\Gin@nat@width\fi}
\def\maxheight{\ifdim\Gin@nat@height>\textheight\textheight\else\Gin@nat@height\fi}
\makeatother
\setkeys{Gin}{width=\maxwidth,height=\maxheight,keepaspectratio}
\makeatletter
\def\fps@figure{htbp}
\makeatother

\renewcommand{\hat}{\widehat}
\renewcommand{\tilde}{\widetilde}

\usepackage{amsthm}
\theoremstyle{plain}

\newtheorem{theorem}{Theorem}[section]
\newtheorem{lemma}[theorem]{Lemma}
\newtheorem{proposition}[theorem]{Proposition}
\newtheorem{corollary}[theorem]{Corollary}
\newtheorem{assumption}{Assumption}
\theoremstyle{definition}
\newtheorem{definition}[theorem]{Definition}

\theoremstyle{remark}

\makeatletter
\@ifpackageloaded{caption}{}{\usepackage{caption}}
\AtBeginDocument{%
\ifdefined\contentsname
  \renewcommand*\contentsname{Table of contents}
\else
  \newcommand\contentsname{Table of contents}
\fi
\ifdefined\listfigurename
  \renewcommand*\listfigurename{List of Figures}
\else
  \newcommand\listfigurename{List of Figures}
\fi
\ifdefined\listtablename
  \renewcommand*\listtablename{List of Tables}
\else
  \newcommand\listtablename{List of Tables}
\fi
\ifdefined\figurename
  \renewcommand*\figurename{Figure}
\else
  \newcommand\figurename{Figure}
\fi
\ifdefined\tablename
  \renewcommand*\tablename{Table}
\else
  \newcommand\tablename{Table}
\fi
}
\@ifpackageloaded{float}{}{\usepackage{float}}
\floatstyle{ruled}
\@ifundefined{c@chapter}{\newfloat{codelisting}{h}{lop}}{\newfloat{codelisting}{h}{lop}[chapter]}
\floatname{codelisting}{Listing}

\makeatother
\makeatletter
\@ifpackageloaded{caption}{}{\usepackage{caption}}
\@ifpackageloaded{subcaption}{}{\usepackage{subcaption}}
\makeatother

\ifLuaTeX
  \usepackage{selnolig}  
\fi

\usepackage[style=asa, sorting=nyt]{biblatex}
\usepackage{bookmark}

\IfFileExists{xurl.sty}{\usepackage{xurl}}{} 
\hypersetup{
  pdftitle={Trustworthy Predictive Distributions for Tail Events with Semiparametric Diagnostic Transport Maps},
  pdfauthor={Elizabeth Cucuzzella; Rafael Izbicki; Ann B. Lee},
  pdfkeywords={predictive uncertainty quantification; local diagnostics; recalibration; conditional density estimation; forecast error distributions},
  colorlinks=true,
  linkcolor={blue},
  filecolor={Maroon},
  citecolor={Blue},
  urlcolor={Blue},
  pdfcreator={LaTeX via pandoc}}

\newcommand{\anon}{1}

\definecolor{vermilion}{rgb}{0.89, 0.26, 0.2}

\renewcommand{\hat}{\widehat}
\renewcommand{\tilde}{\widetilde}

\begin{document}

\def\spacingset#1{\renewcommand{\baselinestretch}%
{#1}\small\normalsize} \spacingset{1}


\if1\anon
{
  \title{\bf Trustworthy Predictive Distributions for Tail Events with Semiparametric Diagnostic Transport Maps}
  \author{Elizabeth Cucuzzella\hspace{.2cm}\\
    Department of Statistics and Data Science, Carnegie Mellon University\\
    and \\
    Rafael Izbicki \\
    Department of Statistics, Federal University of Sao Carlos \\
    and \\
    Ann B. Lee\\
    Department of Statistics and Data Science, Carnegie Mellon University}
  \maketitle
} \fi

\if0\anon
{
  \bigskip
  \bigskip
  \bigskip
  \begin{center}
    {\LARGE\bf Trustworthy Predictive Distributions for Tail Events with Semiparametric Diagnostic Transport Maps}
\end{center}
  \medskip
} \fi

\bigskip
\doublespacing
\begin{abstract} 
Machine learning forecast systems are moving beyond point predictions to full predictive distributions for future outcomes $y$ conditional on complex inputs $x$. However, these distributions are often locally miscalibrated, especially for high-stakes tail events where accurate uncertainty quantification is most needed to establish trust in models. Local miscalibration occurs because training data often lack examples of low-frequency events. The goal of this paper is to describe a simple, yet flexible framework that, at deployment, produces interpretable diagnostics and a robust correction mechanism of predictive distributions when train examples are limited. With this goal in mind, we introduce a semiparametric version of the Local Amortized Diagnostic and Reshaping (LADaR) framework that posits a covariate-dependent parametric model for a \emph{diagnostic transport map} regressed nonparametrically on inputs to describe how to correct tail probabilities across the feature space to match calibration data. These maps provide the user with local, real-time diagnostics and a reshaped predictive distribution that can be related back to physical processes in the input space. We apply these semiparametric diagnostic transport maps to short-term tropical cyclone intensity forecasting to detect evolutionary modes linked to local miscalibration in the National Hurricane Center's forecasts and improve predictions for severe weather hazards.
\end{abstract}

Keywords: predictive uncertainty quantification; local diagnostics; recalibration; conditional density estimation; forecast error distributions

\section{Introduction}
Machine learning research is steadily moving beyond simple point prediction towards methods that also quantify uncertainty. This shift reflects the fact that many real-world decisions require not only a single best guess, but also an understanding of an estimate's reliability. In practice, uncertainty is often summarized through prediction intervals or selected quantiles, commonly trained using scoring rules such as the pinball loss \parencite{gneiting2007, fakoor2023}. These summaries of uncertainties can be useful, but they provide only a partial view of a richer object: {\em the full conditional predictive distribution of future observations}. In particular, prediction intervals and single quantiles do not reveal whether the forecast captures the broader distributional structure correctly, including tail behavior and asymmetry. When decisions depend on these characteristics, the full predictive distribution (PD)---including information about how the distribution of future outcomes varies conditional on specific input sequences---must be reliable and verifiable.

A common strategy in both physics-based forecasting and AI-based generative modeling is to represent uncertainty through an ensemble of point forecasts and then convert that ensemble into a statistical distribution. In numerical weather prediction, for example, multiple runs with perturbed initial conditions produces a collection of plausible outcomes that is then post-processed into a predictive statistical distribution, often in the form of a single point forecast and a parametric ``forecast error distribution'' that is fit to historical data~\parencite{trabing2020}. Similar strategies appear in AI systems, where repeated stochastic queries, sampling-based decoding, tokens or model ensembles are used to generate probabilistic forecasts~\parencite{liu2023, zhou2024, zhao2024}. The latter AI approaches are flexible, but require large amounts of training data, and can be especially unreliable in the tails of the conditional distribution of the response $Y \mid X=x$. Tail events (given $X = x$) are outcomes that lie in the far lower or upper tail of the predictive distribution---such as severe weather, flood, or stock market crash. These events occupy rare, low-frequency regions of the sample space, yet they carry the potential for catastrophic, transformational impact. Models trained primarily on historical data often struggle in these low-frequency regions where tail events lie;  as a result, even when a forecast model outputs a full predictive distribution, it may still misrepresent uncertainty precisely where accurate uncertainty quantification and down-stream risk analysis matters most. Moreover, standard statistical and machine learning objectives, including proper scoring rules~\parencite{gneiting2007jasa}, typically provide only global assessments of predictive performance, averaged over all possible inputs. On their own, they do not identify where in the input space the forecast is locally failing, nor do they reveal how the predictive distribution of $Y$ given $X$ should be adjusted as a function of $X$ across all quantiles $\alpha \in (0, 1)$.

In this paper, we take an initial predictive distribution $\hat F(\cdot\mid x)$ as given and treat it as a useful but potentially misspecified base model. Our goal is twofold:
\begin{itemize}
\item \textbf{(Local diagnostics)} We ask for which covariate values $x$ and outcome levels $y$ the base model $\hat F(\cdot\mid x)$ is locally miscalibrated and how those failures can be computed and visualized in a way that is interpretable by human experts.  That is, we ask {\em where} the model can be trusted and {\em how} tail probabilities change across the feature space.
\item \textbf{(Recalibration)} We then ask whether the base model can be reshaped, at deployment time, into a recalibrated distribution $\widetilde F(\cdot\mid x)$ that is closer to the true conditional distribution, especially with respect to conditional calibration, while preserving the structure already encoded in $\hat F$.  That is, rather than directly fitting target data from  scratch, we ask {\em how} to transport the base model toward a target sample.
\end{itemize}

To address both questions, we introduce \emph{diagnostic transport maps}: probability-to-probability maps that quantify, for each $x$, how the base model's probabilities should be adjusted to better match the true distribution of the target, a.k.a ``calibration'' data. These maps can be
weighted to concentrate correction in the tails of the distributions, targeting the regions of the joint $(x,y)$-space in which the base model is least reliable (See Section 2 of Supplemental Material.) This viewpoint is model-agnostic, computationally efficient at deployment---even for complex, high-dimensional data---and aimed directly at achieving the full predictive distribution rather than a single interval or quantile.

 We show that semiparametric diagnostic transport maps---illustrated in Figure~\ref{fig:kumaraswamyOnNormal}---provide a practical solution to address the problem of unreliable recalibration when calibration data are limited. Whereas the general Local Amortized Diagnostics and Reshaping (LADaR; ~\parencite{dey2025}) framework can handle arbitrary base models (including AI generative models) and be used to diagnose any initial PD, we emphasize interpretability in this paper by starting with a fixed-variance Gaussian distribution, a common approximation in the physical sciences. We posit a parametric family on the probability-to-probability maps, and regress the parameters using a suitable nonparametric regression (e.g. random forest, gradient boosting, neural network, etc.) This semiparametric yet flexible approach increases interpretability and robustness, and makes the method more accessible to general practitioners and students of statistics alike.
 Both in theory and experiments, we show the gains of our semiparametric maps relative to nonparametric LADaR.

As a running application, we consider short-term tropical cyclone intensity forecasting. Accurate probabilistic guidance for rapid 24-hour intensity changes---severe weather hazards driven by rapid shifts in tail probabilities---remains difficult for both physics-based and AI-based systems~\parencite{bi2025}. In this setting, it is especially important for human experts to be able to relate forecasts from different guidance tools to observed physical processes and environmental conditions ahead of a rapid intensity change event. In that setting, our framework yields real-time deployable methods with easily digestible local diagnostics, and a mechanism to immediately produce predictive distributions that improve upon the operational base model from the National Hurricane Center (NHC) given the observed sequence of inputs.

Our paper is organized as follows: in Section~\ref{sec:litReview}, we discuss our work in relation to other methods for uncertainty quantification and diagnostics of predictive distributions. In Section~\ref{sec:basics_of_dot}, we outline the basic set-up, introducing the core ideas of local diagnostics and conditional recalibration. We also show that the diagnostic transport map corresponds to an optimal transport. Furthermore, we demonstrate the robustness of the semiparametric approach and give intuition as to its improvements over the fully nonparametric LADaR. This is followed by an analysis of the theoretical properties of semiparametric diagnostic transport maps in Section~\ref{sec:theory}, which demonstrates the  improved rates of tail events relative to the nonparametric version. Section~\ref{sec:synthetic_example} illustrates our methods on a synthetic example with a known predictive distribution. Finally, Section~\ref{sec:TCApplication} showcases semiparametric diagnostic transport maps on short-term TC intensity forecasting. We conclude with a discussion of our methodology and some future areas of work in spatio-temporal and multivariate predictive distribution calibration.

\begin{figure}
    \centering
    \includegraphics[width=0.85\linewidth]{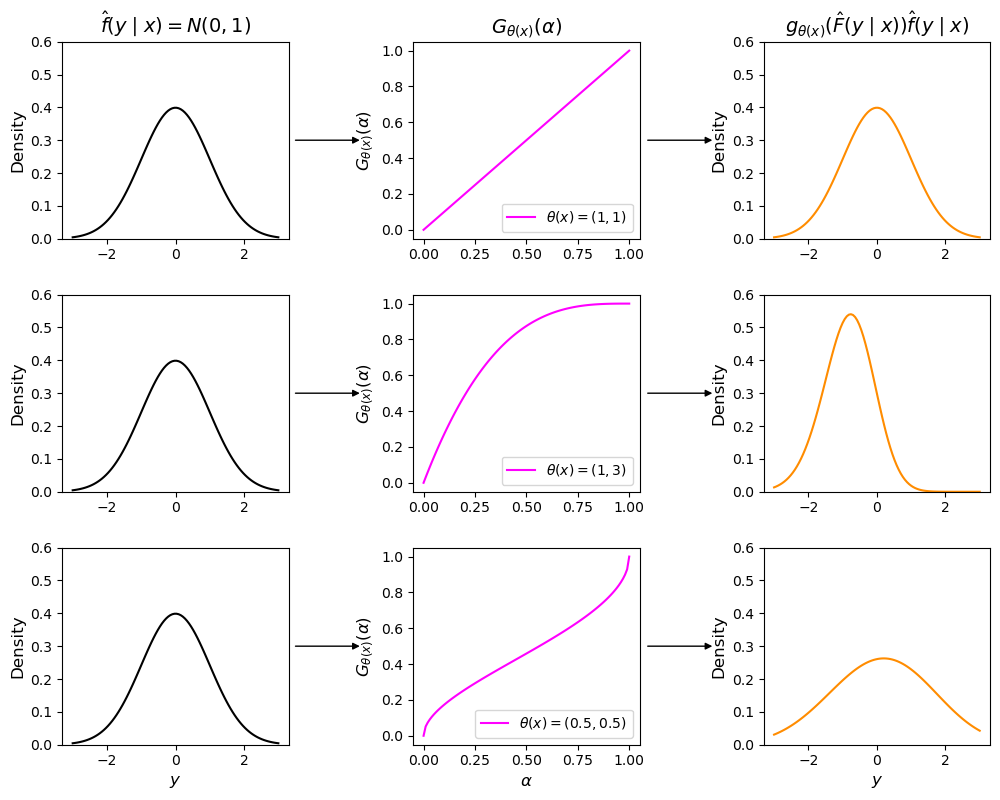}
    \caption{\textbf{Semiparametric diagnostic transport maps.} In diagnostic transport maps, we use learnt diagnostic information from the conditional PIT-CDF $ G(\cdot \mid x)$ with density $g(\cdot \mid x)$ to reshape the base model so as to better reflect the distribution of actual data of interest. In semiparametric maps, we restrict the form of the PIT-CDF to a parametric family $G_\theta$. We then fit the map $x \mapsto \theta(x)$ nonparametrically. The figures show examples of how members of the Kumaraswami family (middle column) would map a standard normal $\hat f$ (left column) to a new density $\tilde f(y \mid x) := g_{\theta(x)}(\hat F(y \mid x)) \hat f(y \mid x)$ (right column) at a fixed $x$. Note that $G_{\theta(x)}(\cdot)$ provides both local diagnostics and the transformation itself. Our algorithm learns this map for all $x$ simultaneously; once the map is learned, one can assess and reshape any base model in real time as input conditions change.}
    \label{fig:kumaraswamyOnNormal}
\end{figure}

\section{Relation to other work}\label{sec:litReview}

\subsection{PIT-based diagnostics}
Classical diagnostics for predictive distributions often rely on the probability integral transform (PIT): if a predictive CDF $\hat F(\cdot\mid x)$ is well calibrated, then $\hat F(Y\mid X)$ should be uniformly distributed~\parencite{gneiting2007}. In practice, however, marginal PIT diagnostics only assess calibration on average and can hide substantial failures in specific parts of the covariate space. The local diagnostic framework of \parencite{zhao2021} addresses this problem by modeling the conditional distribution of the PIT given $X=x$. The framework in \textcite{dey2025} further showed that these local PIT diagnostics can be used to reshape an initial conditional density estimate. The present paper builds on that line of work, but shifts the emphasis to an interpretable and robust semiparametric framework for computing the local PIT diagnostics. This emphasis is motivated by the small-sample regimes that arise for tail events: in precisely those regions of $(x,y)$-space where calibration matters most, calibration data are sparse, and a fully nonparametric estimate of the conditional PIT-CDF can become unstable or noisy.

In this work, we use a parametric transport map at fixed $x$ to impose structure, and to make recalibration feasible even when the relevant events are only weakly represented in the calibration sample. We then use a nonparametric regression with varying coefficients to reveal how the probability-to-probability maps vary across the feature space for all $x$ and for all tail probabilities $(1-\alpha)$ with $\alpha \in [0,1]$.

\subsection{Ensemble post-processing  and recalibration}
\label{sec:postprocessing}
Statistical post-processing of ensemble forecasts shares our goal of correcting
systematic biases and dispersion errors in probabilistic forecasts. In ensemble
model output statistics, a raw ensemble is replaced by a Gaussian
predictive distribution whose mean is a bias-corrected linear combination of
the ensemble members and whose variance is an affine function of the ensemble
variance, with coefficients estimated by minimum CRPS
\parencite{gneiting2005}. The broader literature includes many parametric and
nonparametric approaches \parencite{vannitsem2018,vannitsem2021}.

Our method can be viewed as a form of covariate-dependent post-processing that instead
parameterizes the correction on the probability scale. Taking an initial
predictive CDF $\hat F(\cdot\mid x)$ as given, we estimate its conditional
PIT-CDF $G(\alpha\mid x)$ and define $\widetilde F(y\mid x)=\hat G(\hat F(y\mid x) \mid x)$.
Thus, the contribution is not only covariate-dependent recalibration, but a
representation in which the same map both corrects the forecast and diagnoses
its local misspecification through departures from the identity.

{\subsection{Scoring rules and distributional evaluation}

Proper scoring rules such as the continuous ranked probability score (CRPS) are central tools for evaluating probabilistic forecasts \parencite{gneiting2011,fakoor2023}. Tail-sensitive alternatives include the threshold-weighted CRPS and, more generally, transformed kernel scores for emphasizing high-impact events \parencite{allen2023}. Complementary calibration diagnostics include the tail-calibration framework of \textcite{allen2025}, which assesses the occurrence and severity of threshold exceedances, and the conditional $T$-calibration and reliability diagrams of \textcite{gneitingresin2023}. Our approach is complementary rather than an alternative to this literature: these methods provide targeted tools for evaluating predictive performance and calibration, whereas our conditional PIT-CDF provides an $x$-specific diagnostic of the full predictive distribution and, through its composition with the base CDF, a mechanism for recalibration.

\subsection{Quantile-based methods and tail probabilities}
Quantile regression and related procedures are used to estimate tail quantiles, prediction intervals, and threshold-specific risks  \parencite{chung2021beyond,feldman2021improving}. In particular, the quantile aggregation framework of \parencite{fakoor2023} is close in spirit to our proposed solution in that it allows the analyst to place greater weight on the quantile regions that matter for the problem at hand. The main difference is the statistical target. Quantile methods focus on estimating selected quantiles, or on combining quantile estimates across levels, whereas diagnostic transport maps recalibrate an entire conditional predictive distribution in one step. Once the transport map is estimated, all quantiles of the recalibrated predictive distribution are available simultaneously as a by-product. This removes the need to fit separate models across quantile levels, avoids quantile crossing by construction, and lets the user recover tail probabilities, prediction intervals, and quantiles from the same corrected predictive distribution.

\subsection{Conformal prediction and predictive sets}
Conformal methods are often used to obtain finite-sample valid prediction sets under exchangeability, typically in the sense of marginal coverage \parencite{Vovk2005,lei2018distribution}. More recent variants aim to improve adaptivity across the covariate space and approximate conditional validity \parencite{romano2019conformalizedQR,chernozhukov2021distributional,izbicki2020Dist-Split,izbicki2022,cabezas2025regression}. There is also related work on conformal predictive distributions, including Mondrian conformal predictive distributions \parencite{bostrom2021mondrian}, which apply category-wise calibration across bins of the initial predictions, and recent flow-based conformal predictive distributions~\parencite{harris2026}, which efficiently sample conformal boundaries via flows and then assemble the boundary samples into a single predictive distribution. Diagnostic transport maps are closer to this distributional line of research, but differ in that our methods learn a smooth, covariate-dependent map $\alpha \mapsto G(\alpha\mid x)$ that both recalibrates the initial predictive distribution and serves as an interpretable diagnostic: its shape reveals how the base model fails locally. Prediction sets at any desired level can then be read off from the recalibrated distribution, whereas prediction sets alone do not recover this local structure.

\section{The basics of diagnostic transport maps}\label{sec:basics_of_dot}

\subsection{Set-up and notation}
We consider a supervised learning setting with covariates $X\in \mathcal{X}\subset \mathbb{R}^d$ and a continuous response $Y\in\mathcal{Y}\subset \mathbb{R}$. We denote the true conditional predictive CDF of $Y$ given $X = x$ by $F(y\mid x) := \mathbb{P}(Y\leq y\mid X = x)$. To diagnose and reshape an initial distribution $\hat{F}$, we assume access to a calibration sample $\mathcal{T}_\text{cal} = \{(X_i, Y_i)\}_{i = 1}^N$, which is independent of the training sample used to produce the base model $\hat{F}$.  See Table~\ref{tab:terminology} in Appendix~\ref{sec:notation} for a list of symbols.

\subsection{Local diagnostics based on PIT}
Our approach analyses the distribution of the probability integral transform (PIT) across the covariate space $\mathcal{X}$. Given a fixed initial predictive distribution $\hat{F}(\cdot \mid x)$, we define a PIT variable for each observation $(X,Y)$,
\[Z := \hat{F}(Y\mid X)\]
That is, for each $(X, Y)$, we evaluate the CDF of the base model at the realized outcome $Y$.

For a fixed covariate value $x\in \mathcal{X}$, we then define the \emph{conditional PIT-CDF} as 
\begin{equation} \label{eq:cond-pit-cdf}
G(\alpha \mid x) 
:= \mathbb{P}\big( Z \le \alpha \mid X = x \big),
\end{equation}
where $\alpha \in [0,1]$. Intuitively, $G(\cdot \mid x)$ describes how the PIT values are distributed at covariate value $x$. It follows from Lemma~\ref{lem:true_conditional_pit} that the base model is perfectly calibrated at $x$ if, and only if, $Z\mid X = x$ is Uniform(0, 1); that is,
\[G(\alpha\mid x) = \alpha \text{ for every } \alpha \in [0, 1].\]
It is the \emph{dependence on $x$} that is interesting in regression and prediction problems, and that standard PIT and other popular diagnostics approaches do not model.

\subsection{Recalibrating the base model across the covariate space}
Next we address the question of how to map the initial distribution $\hat{F}(\cdot \mid x)$ to the ``target'' distribution $F(\cdot\mid x)$ of the calibration sample. The key observation underlying our proposed approach is that the true predictive CDF
\(F(\cdot \mid x)\) can be written as a composition of the CDF of the known base model
\(\hat F(\cdot\mid x)\) and the (in practice, unknown) true conditional PIT-CDF \(G(\cdot\mid x)\). The following lemma formalizes this oracle representation: \\

\begin{lemma}
\label{lem:oracle-intro}
Fix \(x \in \mathcal{X}\) and assume that the initial predictive CDF
\(\hat F(\cdot \mid x)\) is strictly increasing and continuous in \(y\).
For every
\(y \in \mathcal{Y}\),
\[
F(y \mid x) = \mathbb{P}(Y \le y \mid X=x)
= G\big(\hat F(y \mid x) \mid x \big).
\]
\end{lemma}

Lemma \ref{lem:oracle-intro} shows that recalibrating the base CDF model $\hat{F}(\cdot \mid x)$ is equivalent to estimating a family $\{G_x : x \in \mathcal{X} \}$ of {\em diagnostic transport maps} 
\begin{equation}\label{eq:diag-transport-maps}
    G_x : \alpha \mapsto G(\alpha \mid x)
\end{equation}
with $G_x(\alpha):=G(\alpha\mid x)$, which transforms the original probabilities into calibrated probabilities for every $\alpha \in [0,1]$ and $x \in \mathcal{X}$. 

By regression, we can obtain an estimate of the conditional PIT-CDF $G(\cdot\mid x)$ (Equation~\ref{eq:cond-pit-cdf}) and hence an estimate of the entire family of transport maps for $x \in \mathcal X$. This can be done via a semiparametric approach (see Section~\ref{subsec:semi}) or---if there is ample calibration data----via a fully nonparametric probabilistic classification and e.g. monotone neural networks (see Section 3.2 in \textcite{dey2025}.) In either case, we define the recalibrated predictive distribution as follows: \\

\begin{definition} [Recalibrated predictive distribution]\label{def:recalibrated_PD}
  The recalibrated predictive CDF of $Y$ given $x$ is defined by
 \begin{equation}
\widetilde F(y|x):= \hat G_x \left(\hat F(y|x)\right), 
 \label{eq:recalibrated_PD}
 \end{equation} 
 where $\hat G_x$ 
 is an estimate of $G_x$.
\end{definition}

When both the original predictive distribution and the PIT model admit densities, the calibrated \emph{density} can also be obtained in closed form. Let $\hat{f}(y\mid x) := \partial_y\hat{F}(y\mid x)$ denote the initial conditional density of $Y\mid X = x$, and let $g(\alpha):= \partial_\alpha \hat{G}_x(\alpha)$ denote the density of $Z$ under the PIT model $\hat{G}_x$. Differentiating the composition $\tilde{F}(y\mid x) = \hat{G}_x(\hat{F}(y\mid x))$ with respect to $y$ yields, by the chain rule, 
\begin{equation}
\tilde{f}(y\mid x):= \partial_y\tilde{F}(y\mid x) = \hat{g}_x(\hat{F}(y\mid x))\hat{f}(y\mid x)  
\end{equation}
That is, the transport map correction rescales the initial density estimate by the PIT density evaluated at the base model CDF value.

\begin{figure}
    \centering
    \includegraphics[width=1.0\linewidth]{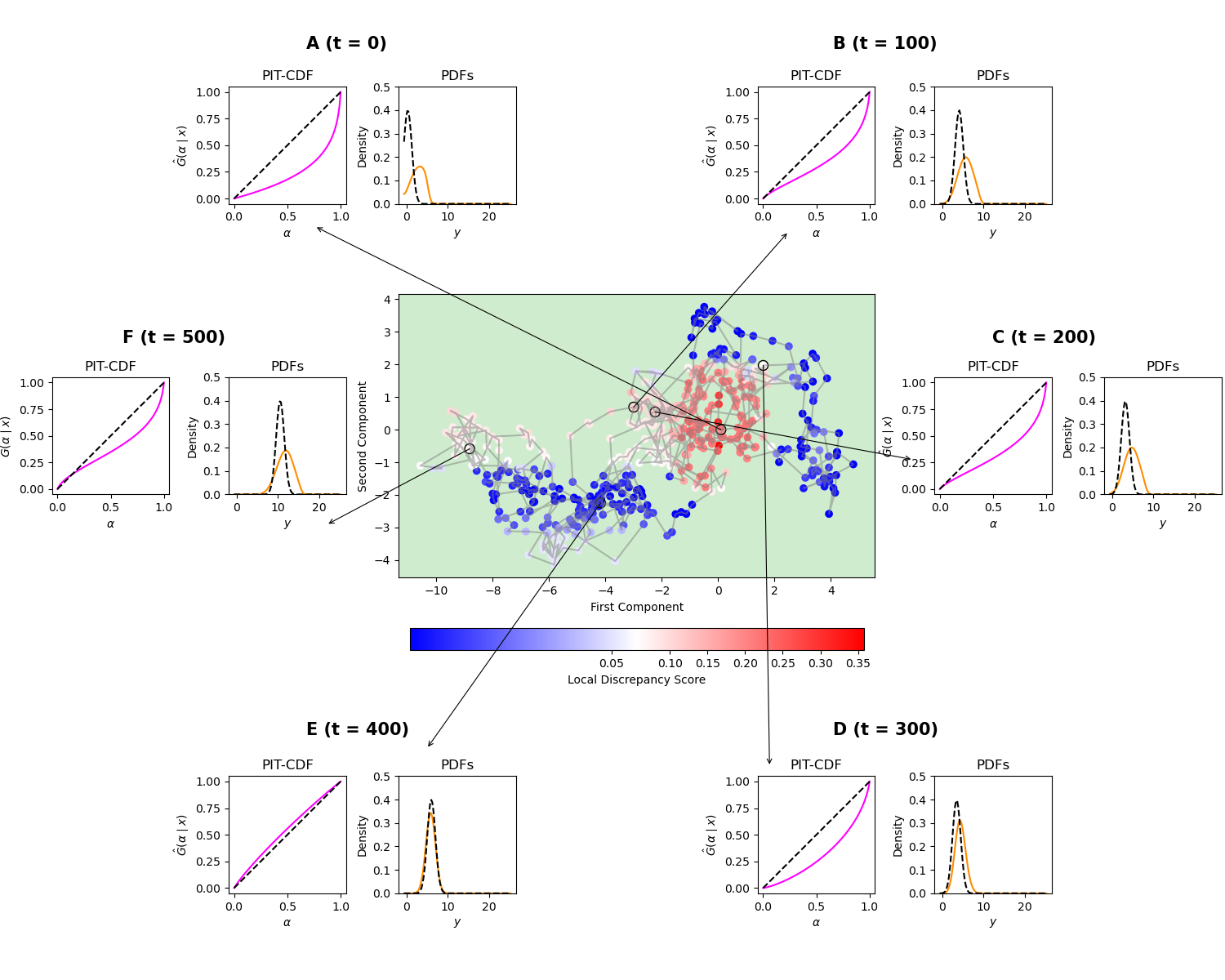}
    \caption{\small{{\bf Synthetic example.} The true (unknown) predictive distribution $F(y_t\mid x_t)$, where $x_t \in \mathbb{R}^2$, follows a Gumbel distribution with a shape that varies with $||x_t||_2$; see Section~\ref{sec:synthetic_example} for details. From calibration data, we estimate a diagnostic map that transports a Gaussian base model to a reshaped distribution that better approximates the predictive distribution for all values of $x_t$ and $y_t$. The central image in the figure depicts the test points in the input space $\mathcal{X}$ with the color of each point encoding the estimated local discrepancy score (Equation~\ref{eq:lds_equation}); this score provides the scientist with guidance on how far the original base model is from the true predictive distribution for different $x_t$-values. The estimated PIT-CDF diagnostics (here semiparametrically estimated) provide detailed information on the failure; each matching plot to the right shows the base (dotted black) and reshaped PDFs (orange) before and after applying the transport map. }}
    \label{fig:synthetic}
\end{figure}

Figure~\ref{fig:synthetic} illustrates the key idea and utility of diagnostic transport maps $G_X$.
These maps (i) pinpoint where in the feature space $\mathcal{X}$ the base model is a poor fit, (ii) directly convey the failure mode (bias, dispersion, skewness, tail weight, etc.) via PIT-CDF diagnostic plots (see Appendix~\ref{appendix:diagnostics} and Figure~\ref{fig:cheatsheet} for details), and (iii) offer a mechanism to recalibrate the predictive distribution. No other approach in the literature offers a general unified framework to local diagnostics and recalibration of PDs, which can be adapted to any type of data $x$ and initial predictive distribution.

Finally, forecasters are often interested in the quality of estimating the probability of a binary event $\{Y \leq t \}$ at some pre-determined threshold $t \in \mathbb{R}$; for example, TC forecasters are interested in estimating the probability that a storm's intensity is less than or equal to hurricane status (63 knots) at time $t+24$ hours. Corollary~\ref{cor:thresholded-forecasts} directly relates the error in (thresholded) probability forecasts to the error in the PIT-CDF estimates:\\

\begin{corollary}[Error in probability forecasts]
\label{cor:thresholded-forecasts}
Let $F(t \mid x) = \mathbb{P}(Y\leq t\mid x)$ with initial estimate $p=\hat{F}(t\mid x)$.  For every fixed $x \in \mathcal{X}$ and threshold $t \in \mathbb{R}$,
\[ \left|\tilde F(t \mid x) - {F}(t \mid x) \right| = \left|\hat G_x(p ) - G_x(p ) \right|\]
\end{corollary}
As a consequence, we can also rewrite the integrated squared error (ISE) of the reshaped CDF at fixed $x$ in terms of  the conditional PIT-CDF errors: 
 \begin{align}
    \mathcal{E}(x) &:=\int_{-\infty}^{\infty}\big(\tilde{F}(t\mid x)-F(t\mid x)\big)^2dt \nonumber  \\
    & = \int_{0}^{1} \frac{\big(\hat{G}(p\mid x)-G(p\mid x)\big)^2} {\hat{f} (u|x)} dp, \label{eq:true_ISE}
\end{align}    
where $u= \hat{F}^{-1}(p|x)$. The ISE is directly related to two standard proper scoring rules: the continuous ranked probability score (CRPS; Definition 1.1 in Supplemental Materials) commonly used in conditional density estimation, and the pinball loss function commonly used in quantile regression. See Supplemental Materials, Section 2, for details.

\subsubsection{An algorithmic solution to estimating a family of optimal transport maps}
The optimal transport (OT), which rearranges a source distribution $\hat{F}(y\mid x)$ into a target distribution $F(y\mid x)$ in outcome space for fixed $x$ is given by $T_x(y) = F^{-1}(\hat{F}(y\mid x)\mid x)$ \parencite{santambrogio2015optimal}. The diagnostic transport map corresponds to an OT map and can be seen as an algorithmic solution to estimating the entire family of OT maps across a potentially high-dimensional feature space via a \emph{single} regression.\\

\begin{proposition}[OT map as a composition involving \(G\)]
\label{prop:Tx-G}
Fix \(x \in \mathcal{X}\) and assume that
\(F(\cdot\mid x)\), \(\hat F(\cdot\mid x)\), and \(G(\cdot\mid x)\) are
continuous and strictly increasing. Then, for any \(y\in\mathcal{Y}\),
\[
T_x(y)
= F^{-1}\!\big(\hat F(y\mid x)\mid x\big)
= \hat F^{-1}\!\Big(G^{-1}\big(\hat F(y\mid x)\mid x\big)\,\Bigm|\,x\Big).
\]
\end{proposition}

Figure~\ref{fig:OTExplanation} illustrates the parallel views of diagnostic transport maps and optimal transport (see  Sections 3-4 Supplemental Materials for further details).
\begin{figure}
    \centering
    \includegraphics[width=0.65\linewidth]{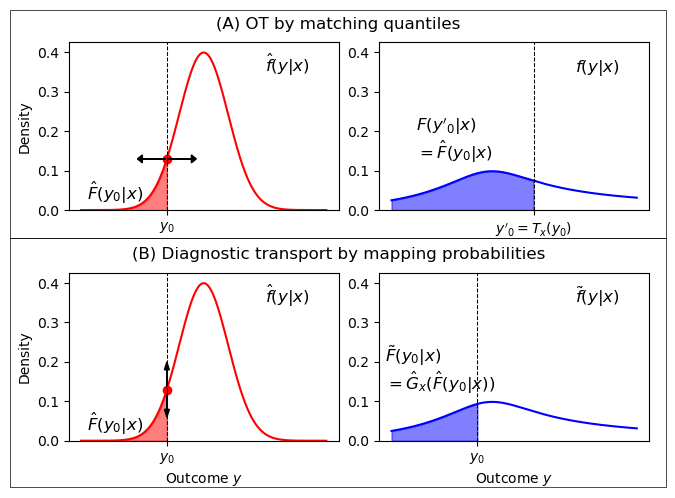}
    \caption{\small{{\bf Diagnostic maps as the solution to an optimal transport problem.}\emph{Top (A):} A standard optimal transport (OT) map from $\hat{F}(y\mid x)$ to $F(y\mid x)$ uses quantile matching for fixed $x$ to match calibration data, where the OT map is given by $T_x(y_0) := F^{-1}(\hat{F}(y_0\mid x)\mid x).$ The OT map rearranges one PDF into another in outcome space, as indicated by the left-right double arrows. In practice, the true target distribution $F$ is not known, and an OT map is also not designed to provide diagnostics of a base model. \emph{Bottom (B):} A diagnostic transport constructs an estimate of the entire conditional distribution $F(y\mid x)$ by mapping probabilities, yielding the recalibrated distribution $\tilde{F}(y_0\mid x) = \hat G_{x}(\hat F(y\mid x))$ (Definition~\ref{def:recalibrated_PD}). We achieve an estimate of the entire family of OT maps via $\hat{T}(y_0\mid x) := \tilde{F}^{-1}(\hat{F}(y_0\mid x)\mid x)$ for all $x\in \mathcal{X}$ and $ y_0 \in \mathcal{Y}$ (Proposition~\ref{prop:Tx-G}) via a single regression. That is, the diagnostic map reshapes the original PDF in probability space, as indicated by the top-bottom double arrows, but the end result of matching quantiles or mapping probabilities is the same at every $x$.}} 
    \label{fig:OTExplanation}
\end{figure}

\subsection{Robustness via semiparametric diagnostic transport maps}\label{subsec:semi}
The original calibration algorithm by \textcite{dey2025} estimates $G(\cdot \mid x)$ nonparametrically using a monotonic neural network that regresses the binary event $\mathbbm{1}\{\hat F(Y \mid X) \leq \alpha \}$ on $(\alpha, X)$. While a fully nonparametric approach is flexible and asymptotically consistent, it can produce noisy or unstable predictive distributions when the calibration sample size $N$ is moderate or small; this is especially a problem for the tails of the predictive distribution that are not well-represented in calibration data. 
To better handle tail events, we propose a \emph{semiparametric} version of diagnostic transport maps that posits a parametric model on the conditional PIT-CDF at every fixed $x$, and then regresses the PIT-CDF nonparametrically with respect to $x$ via a varying-coefficient-type model.

\paragraph*{Local PIT parametrization.} In many scientific applications, calibration data are often limited in size, making it difficult to solve for $\hat{G}(\cdot \mid x)$ non-parametrically. Therefore, instead of learning an arbitrary function $(\alpha, x)\mapsto G(\alpha\mid x)$, we posit that for each $x\in \mathcal{X}$ the conditional distribution of $Z\mid X=x$ belongs to a parametric family on $[0,1]$, indexed by a \emph{local parameter} $\theta(x)\in \Theta\subset \mathbb{R}^p$:
\[G(\alpha\mid x)\approx G_{\theta(x)}(\alpha)\]
Here, $G_\theta$ denotes a CDF on $[0,1]$ for each $\theta \in \Theta$ that includes the Uniform(0,1) distribution as a special case (for example, $G_\theta$ could be the Beta distribution.) In this work, we use the two-parameter Kumaraswamy distribution~\parencite{kumaraswamy1980}, which is a Beta-type distribution with closed-form expressions for its CDF, PDF, and quantile function. The Kumaraswamy distribution is a rich family of density functions that is easy to approximate and unlikely to lead to numerical instabilities in the tails (more specifically, it satisfies relative tail stability according to Assumption~\ref{ass:relative-tail-stability}). Figure~\ref{fig:kumaraswamyOnNormal} shows some examples of how members of the Kumaraswamy distribution reshape a standard normal.

Given the calibration sample \(\mathcal{T}_{\rm cal} = \{(X_i,Y_i)\}_{i=1}^N\),
we form the PIT values
\[
Z_i := \hat F(Y_i \mid X_i), \qquad i=1,\dots,N.
\]

\paragraph*{PIT regression.}
Next we fit the function  $x\mapsto \theta(x)$ non-parametrically by (for example) maximizing the local conditional likelihood of \(Z_i \mid X_i=x \sim G_{\theta(x)}\). 
The result is an estimated parameter function
\(\widehat\theta(x)\) and an estimated PIT-CDF
\[
\widehat G(\alpha \mid x) := G_{\widehat\theta(x)}(\alpha).
\] 

The recalibrated predictive CDF for \(Y\mid X=x\) is obtained by
the composition 
\begin{equation}
\widetilde F(y \mid x)
:= \widehat G\bigl(\hat F(y\mid x) \mid x\bigr)
= G_{\widehat\theta(x)}\bigl(\hat F(y\mid x)\bigr).
\end{equation}

Intuitively, it is clear why a semiparametric approach is more robust than a fully nonparametric approach when estimating tail events: whereas a nonparametric method only uses information from the tails of the conditional distribution of $Y \mid X$ (hence large sample sizes are needed), a semiparametric method leverages all values of $Y \mid X$  (at the cost of bias due to the parametric shape). We will make this rigorous in the next section. Although not demonstrated in the examples in this paper, one could obtain uncertainties on the estimated PIT-CDF itself via bootstrapping $\hat{\theta}(x)$.

\section{Theory of semiparametric diagnostic transport maps}\label{sec:theory}
As before, let $\hat F(\cdot\mid x)$ be the
base predictive CDF, $Z=\hat F(Y\mid X)$ be the PIT variable, and
$G(\alpha\mid x)=\mathbb P(Z\le\alpha\mid X=x)$  the conditional PIT-CDF.
Throughout this section,  $\hat F$ is treated as a
fixed function.
We further assume that the map $y\mapsto\hat F(y\mid x)$ is strictly
increasing and continuous for every $x$.   
The fully nonparametric LADaR estimator estimates the   bivariate
surface $(\alpha,x)\mapsto G(\alpha\mid x)$.
The semiparametric variant restricts the $\alpha$-direction to a
low-dimensional family of CDFs on $[0,1]$,
$
\mathcal G=\{G_\theta:\theta\in\Theta\subset\mathbb R^p\},
$
and approximates the conditional PIT-CDF as
$
G(\alpha\mid x)\approx G_{\theta(x)}(\alpha)$, and thus $\tilde F(y\mid x)=G_{\hat\theta(x)}\{\hat F(y\mid x)\}.
$ 

To fit the PIT family, let   $L(z,\theta)$ be a loss function.  For each
$x\in\mathcal X$, define the oracle PIT parameter by the pointwise
conditional risk minimizer  $$
\theta^\star(x)
= 
\arg\min_{\theta\in\Theta}
\mathbb E\{L(Z,\theta)\mid X=x\}.
$$
The corresponding oracle semiparametric PIT map is
$
G^\star(\alpha\mid x)
:=
G_{\theta^\star(x)}(\alpha).
$
If the true conditional PIT
distribution belongs to the family $\mathcal G$ at $x$, then strict
propriety of the loss implies
$
G_{\theta^\star(x)}(\cdot)=G(\cdot\mid x).
$

\subsection{Error decomposition}

Let $\widehat \theta(x)$ be an estimate of $\theta(x)$.
The total error of the semiparametric recalibrated CDF decomposes into
two interpretable components: 

\begin{proposition}[Error decomposition]
\label{prop:stats-decomp}
For every $x\in\mathcal X$ and $y\in\mathbb R$,
$$
\begin{aligned}
|\tilde F(y\mid x)-F(y\mid x)|
&\le
\underbrace{|G_{\theta^\star(x)}\{\hat F(y\mid x)\}-G\{\hat F(y\mid x)\mid x\}|}_{\text{semiparametric approximation error}}\\
&\quad+
\underbrace{|G_{\hat\theta(x)}\{\hat F(y\mid x)\}-G_{\theta^\star(x)}\{\hat F(y\mid x)\}|}_{\text{estimation error}}.
\end{aligned}
$$
In particular,
$
\sup_y|\tilde F(y\mid x)-F(y\mid x)|\le A_\infty(x)+E_\infty(x),
$
with $A_\infty(x)=\sup_{\alpha\in[0,1]}|G_{\theta^\star(x)}(\alpha)-G(\alpha\mid x)|$ and $E_\infty(x)=\sup_{\alpha\in[0,1]}|G_{\hat\theta(x)}(\alpha)-G_{\theta^\star(x)}(\alpha)|$.
\end{proposition}

The first term (``semiparametric approximation error'') captures the irreducible bias imposed by the choice of $G$. The chosen PIT family may not contain the true conditional PIT-CDF $G(\cdot \mid x)$, which leaves a non-vanishing bias term even with infinite calibration data.

\subsection{Error rates}

In this section, we show how a  
pointwise rate for $\hat\theta(x)$  propagates to
the recalibrated CDF.  We first impose a regularity condition on the chosen
PIT family:

\begin{assumption}[Lipschitz PIT family]
\label{ass:lipschitz-stats}
There exists $K<\infty$ such that
$$
\sup_{\alpha\in[0,1]}|G_{\theta_1}(\alpha)-G_{\theta_2}(\alpha)|\le K\|\theta_1-\theta_2\|
\qquad\text{for all }\theta_1,\theta_2\in\Theta.
$$
\end{assumption}

This condition holds on compact $\Theta$  for the Beta and the Kumaraswamy families.
We also assume a standard nonparametric rate $r_N$ for the estimator $\hat\theta$:

\begin{assumption}[Varying-coefficient rate]
\label{ass:theta-rate-stats}
For a deterministic sequence $r_N(x)\to0$,
$
\|\hat\theta(x)-\theta^\star(x)\|=O_P\{r_N(x)\}.
$
\end{assumption}

When the population target $\theta^\star(\cdot)$ is $\beta$-smooth in
$d$ covariates, the rate
$r_N(x)=N^{-\beta/(2\beta+d)}$
is typically attained by  standard nonparametric estimators
tuned appropriately  \parencite{newey1994,shen1994,chen2007}.
These assumptions lead to the following rate:

\begin{proposition}
\label{prop:stats-rate}
Under Assumptions~\ref{ass:lipschitz-stats}--\ref{ass:theta-rate-stats},
$
\sup_y|\tilde F(y\mid x)-F(y\mid x)|\le A_\infty(x)+O_P\{r_N(x)\}.
$
Thus, if the semiparametric PIT family is correctly specified at $x$, in
the sense that
$
G_{\theta^\star(x)}(\cdot)=G(\cdot\mid x),
$ 
then $A_\infty(x)=0$, and $\tilde F(\cdot\mid x)$ inherits the rate
$r_N(x)$ of $\hat\theta(x)$.
\end{proposition}

Thus, after the semiparametric approximation error $A_\infty(x)$ is
separated out, the stochastic part of the recalibration error is governed by
the smoothness of the low-dimensional coefficient function
$\theta^\star(\cdot)$.  This differs from direct nonparametric estimation
of the full  distribution $F(y\mid x)$, or of the full PIT surface
$G(\alpha\mid x)$, both of which involve an additional response or
probability coordinate unless further structure is imposed. If the covariates have lower-dimensional structure and if the estimator is properly chosen, then $d$ denotes the intrinsic rather than ambient dimension \parencite{bickel2007local,kpotufe2011knn,kpotufe2013adaptivity,lee2016spectral}. 

\subsection{Comparison with nonparametric recalibration for tail events}
\label{sec:tail-comparison}

We compare the semiparametric recalibration estimator with a fully
nonparametric local estimator in the lower tail.  Fix $x\in\mathcal X$ and
a tail probability $\alpha_N$.  Define
$
p_N(x)=G(\alpha_N\mid x)
      =\mathbb P(Z\le \alpha_N\mid X=x),
$
and assume $p_N(x)>0$ for all sufficiently large $N$.

\textbf{Nonparametric estimators.} 
A standard nonparametric estimator of $p_N(x)$ is the local kernel
frequency estimator
$$
\hat p_{{\rm np},h_N}(x)
=
\frac{
\sum_{i=1}^N K_{h_N}(x-X_i)\mathbb{I}\{Z_i\le \alpha_N\}
}{
\sum_{i=1}^N K_{h_N}(x-X_i)
},
$$
where $K_{h_N}(u)=h_N^{-d}K(u/h_N)$ and $h_N\to0$.  Under standard
random-design and kernel regularity conditions, and ignoring smoothing
bias, \textcite{daouia2013}, Proposition 1 gives
$$
\sqrt{N h_N^d p_N(x)}
\left\{
\frac{\hat p_{{\rm np},h_N}(x)}{p_N(x)}-1
\right\}
=
O_P(1),
$$
or equivalently
$$
\frac{\hat p_{{\rm np},h_N}(x)-p_N(x)}{p_N(x)}
=
O_P\!\left\{(N h_N^d p_N(x))^{-1/2}\right\}.
$$
Thus, apart from smoothing bias, relative consistency requires
the sequence $p_N$ to decrease not too fast:
$
N h_N^d p_N(x)\to\infty,
$
that is,
$
p_N(x)\gg (N h_N^d)^{-1}.
$
In particular, the usual bandwidth
$ h_N\asymp N^{-1/(2\beta+d)}$ 
requires
\begin{align}
\label{eq:nonpar}
p_N(x)\gg N^{-2\beta/(2\beta+d)}.
\end{align}
Relative consistency for the nonparametric method therefore  
does not allow $p_N(x)$ to go to zero too fast; it
needs many observations that are both
local in $\mathcal X$ and in the tail event $\{ Z\le\alpha_N \}$.

\textbf{Semiparametric estimators.} 
The semiparametric estimator of the same tail probability is
$
\hat p_{\rm sp}(x)=G_{\hat\theta(x)}(\alpha_N).
$ 
To analyse such estimator, we impose the following structural stability condition on the fitted PIT
family (not on the true data-generating process):

\begin{assumption}[Relative tail stability]
\label{ass:relative-tail-stability}
There exist constants $C<\infty$, $\rho>0$, and $\alpha_0>0$, and a
nonnegative function $S(\alpha,x)$, such that for every
$\alpha\in(0,\alpha_0)$ and every $\theta$ satisfying
$$
\|\theta-\theta^\star(x)\|\le \frac{\rho}{S(\alpha,x)},
$$
it holds that
\begin{equation}
\label{eq:relative-tail-stability}
\frac{
\left|G_{\theta}(\alpha)-G_{\theta^\star(x)}(\alpha)\right|
}{
G_{\theta^\star(x)}(\alpha)
}
\le
C\,S(\alpha,x)\|\theta-\theta^\star(x)\|.
\end{equation}
\end{assumption}

The function $S(\alpha,x)$ measures how strongly the parameter error is
amplified when the fitted PIT map is evaluated in the lower tail. 

\begin{proposition}[Semiparametric relative tail error]
\label{prop:sp-tail-rate}
Fix $x\in\mathcal X$ and let $\alpha_N\downarrow0$.  Define
$
A_N(x)
=
\left|
G_{\theta^\star(x)}(\alpha_N)-G(\alpha_N\mid x)
\right|.
$ 
Suppose Assumptions~\ref{ass:theta-rate-stats} and
\ref{ass:relative-tail-stability} hold.  If
$
S(\alpha_N,x)r_N(x)\to0,
$
then the relative error is
$$
\frac{|\hat p_{\rm sp}(x)-p_N(x)|}{p_N(x)}
\le
\frac{A_N(x)}{p_N(x)}
+
O_P\!\left[
S(\alpha_N,x)r_N(x)
\left\{
1+\frac{A_N(x)}{p_N(x)}
\right\}
\right].
$$
Consequently, if $A_N(x)=O\{p_N(x)\}$, then
$$
\frac{|\hat p_{\rm sp}(x)-p_N(x)|}{p_N(x)}
\le
\frac{A_N(x)}{p_N(x)}
+
O_P\{S(\alpha_N,x)r_N(x)\}.
$$
In particular, relative consistency holds whenever
$
A_N(x)=o\{p_N(x)\}$  and $S(\alpha_N,x)r_N(x)\to0.
$
\end{proposition}

For the Beta and Kumaraswamy families,
Assumption~\ref{ass:relative-tail-stability} holds with
$
S(\alpha,x)=1+|\log\alpha|
$
under compact shape-parameter restrictions, and thus the semiparametric stochastic
condition is
$
\{1+|\log\alpha_N|\}r_N(x)\to0.
$
Under the standard rate $r_N(x)=N^{-\beta/(2\beta+d)}$, this becomes
$
|\log\alpha_N|
=
o\!\left(N^{\beta/(2\beta+d)}\right).
$
Thus the semiparametric estimator can target nominal PIT levels as small as
$$
\alpha_N
=
\exp\!\left\{
-o\!\left(N^{\beta/(2\beta+d)}\right)
\right\},
$$
provided the oracle semiparametric tail approximation is relatively
negligible, i.e., $A_N(x)=o\{p_N(x)\}$.

Now, assume further that the conditional PIT-CDF has a regular lower tail, i.e.,
$G(\alpha\mid x)=c(x)\alpha^{\gamma(x)}L_x(\alpha)$ for small $\alpha$ \parencite{chen2007,paulauskas2017}, 
where $c(x)>0$, $\gamma(x)>0$, and $L_x$ is slowly varying.  Then
$|\log p_N(x)|$ and $|\log \alpha_N|$ are of the same order (see Appendix~\ref{app:proofs} for formal definitions and proofs.) Hence, for the
Beta and Kumaraswamy families and the standard coefficient-function rate
$r_N(x)=N^{-\beta/(2\beta+d)}$, relative consistency requires 
\[
|\log p_N(x)|
=
o\!\left(N^{\beta/(2\beta+d)}\right).
\]

By contrast, the fully nonparametric local-frequency estimator  
requires $
p_N(x)\gg N^{-2\beta/(2\beta+d)}
$ (Eq. \ref{eq:nonpar}).
Thus, under this benchmark comparison, nonparametric recalibration requires
the target tail probability to exceed a polynomial threshold, whereas the
Kumaraswamy semiparametric estimator can target probabilities
that go to zero at a much faster rate,
\[
p_N(x)
=
\exp\!\left\{
-o\!\left(N^{\beta/(2\beta+d)}\right)
\right\},
\]
as long as the relative oracle tail
approximation $A_N(x)=o\{p_N(x)\}$.
 In words: the nonparametric estimator is assumption-light about the shape of
the PIT tail, but it relies on local tail counts.  Hence relative consistency
requires enough calibration points near $x$ falling in $Z\le\alpha_N$, namely
$N h_N^d p_N(x)\to\infty$, so $p_N(x)$ cannot vanish too quickly.  The
semiparametric estimator  alleviates the curse of dimensionality  by replacing this counting with model structure:
it can handle much smaller tail probabilities, but only when the chosen PIT
family gives a sufficiently accurate relative tail approximation.

\section{Synthetic Example}\label{sec:synthetic_example}

\paragraph*{Set-up.} To illustrate our method, we consider a time-homogeneous conditional distribution $Y_t\mid X_t$ with time series data $\{ (X_t, Y_t) \}_{t \geq 0}$, where $X_t \in \mathbb{R}^2$, $Y_t \in \mathbb{R}$ and $t$ is a non-negative integer. The predictor values are given by a Gaussian random walk: $X_{t+1}=X_{t}+\phi $, where $\phi \sim N(0,\sigma^2)$ represents Gaussian white noise with a fixed variance $\sigma^2$. The true (but unknown) predictive CDF $F(y_t \mid x_t)$ is given by a Gumbel distribution~\parencite{gumbel941} with location and scale parameters $\mu = ||x_t||_2$ and $\beta = \sqrt{||x_t||_2}$, respectively. Our base model $\hat F(y_t \mid x_t)$ is a normal distribution with a fixed unit variance and the same mean as the Gumbel distribution; that is, $\hat{F}(y_t\mid x_t)=N(\mu_G(x_t), 1)$, where $\mu_G(x_t) = ||x_t||_2 + \sqrt{||x_t||_2} \gamma$ and $\gamma$ is the Euler's constant.

\paragraph*{Diagnostic transport maps.} 
We model the conditional PIT-CDF $G(\alpha \mid X_t=x)$ at a fixed $x$-value with a two-parameter Kumaraswamy distribution~\parencite{kumaraswamy1980}
\begin{equation}\label{eq:kumaraswamy}
G_{a,b}(\alpha) = 1 - (1 - \alpha^a)^b
\end{equation}
where $a > 0, b >0$. Let $\mathcal{T}_\text{cal}=\{ (X_t, Y_t) \}_{t=0}^{N-1}$ denote the calibration data. For a {\em semiparametric} diagnostic transport map, we first estimate $\theta_t = (a_t, b_t)$ for $t=0, \ldots, N-1$ by minimizing the continuous ranked probability score (CRPS; Section 1, Supplemental Materials) locally for each point's $k=15$  nearest neighbors in $x$, where $k$ was selected using five-fold cross-validation on the calibration data. Then we regress the parameter estimates $\hat \theta_t$ on $X_t$ via gradient boosting to learn the map $x \mapsto \theta(x)$. The regression gives us the estimated PIT-CDF, $\hat G(\cdot \mid x) = G_{\hat \theta(x)}(\cdot)$.  Finally, to compare with a fully {\em nonparametric} version of diagnostic transport maps, we also fit $G(\cdot \mid x)$ with the algorithm developed in \textcite{dey2025} using a gradient boosting classifier that is monotonic with respect to $\alpha$.

\paragraph*{Results.}  Figure~\ref{fig:synthetic} shows examples of local PIT-CDFs learned via semiparametric diagnostic transport maps, together with the corresponding predictive PDFs before and after recalibration, for test data. The color of the points here codes for the ``local discrepancy score'' (LDS; Appendix~\ref{appendix:diagnostics}3)
\begin{equation}\label{eq:lds_equation}
    \text{LDS}(x) := \frac{1}{|\Gamma|}\sum_{\alpha\in \Gamma}(\hat{G}(\alpha\mid x) - \alpha)^2 ,
\end{equation}
where $\Gamma$ is a set of evenly spaced points in the interval $(0,1)$.
This score measures the discrepancy between the estimated PIT-CDF and the uniform distribution at a fixed value of $x \in \mathcal{X}$. Once we have learned $\hat G$, the evaluation of the diagnostic transport map at different $x$-values is instantaneous. The data analyst can easily produce similar visualizations to assess the base model and reshape predictive PDFs to match calibration data. 

\paragraph*{Comparison between semiparametric and non-parametric maps.} Because this is a synthetic example, we can derive the true conditional PIT-CDF, $G(\alpha \mid {x})=F\left(\Phi^{-1}(\alpha) \mid  {x}\right)$ (see Appendix~\ref{app:proofs} for derivation) 
and then compute the integrated squared error (ISE) of the estimated semiparametric and nonparametric transport maps according to Equation~\ref{eq:true_ISE}. 
We are interested in the convergence rate of the maps, especially in the tails of the distribution of $Y|X$. Using Monte Carlo simulation, we create $B = 25$ $\mathcal{T}_\text{cal}$ data sets at calibration sample sizes $N=25, 50, 75, 100$, and evaluate our results on a fixed grid of $x_t$ values.   Figure~\ref{fig:ConvergenceOfISE} shows that the semiparametric approach results in a lower ISE and faster convergence rates than the nonparametric approach, with the greatest improvements observed for smaller calibration sizes $N$ and in the upper tail region. On a 13th generation Intel core processor with 64-bit CPU, it takes an average of 0.38 seconds to fit the semiparametric map across all $B$ iterations for $N=100$; at inference, semiparametric recalibration takes $6.5 \times 10^-3$ seconds. The corresponding numbers for nonparametric maps are  1.85 seconds for fitting, and 1.56 seconds for recalibration.

\begin{figure}
    \centering
    \includegraphics[width=0.7\linewidth]{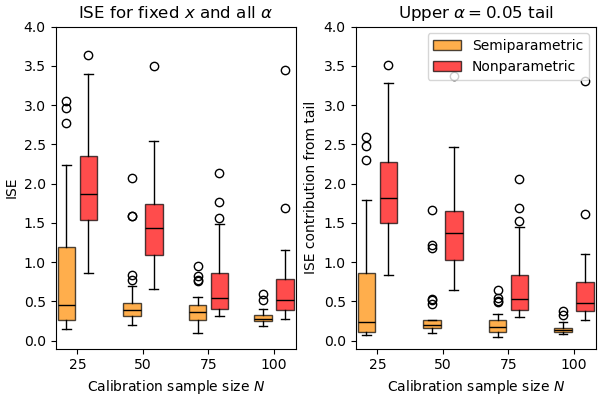}
    \caption{\small{{\bf Performance of semiparametric versus nonparametric diagnostic transport maps.}  {\em Left:} The integrated squared error, ISE, over all $\alpha$ as a function of the calibration sample size $N$. The box plots show the results for the semiparametric versus the nonparametric approach for $B=25$ Monte Carlo simulations each,  evaluated at a fixed $x$ (Example A in Figure~\ref{fig:synthetic}). {\em Right:} Contribution to the ISE from the upper $\alpha = 0.05$ tail. Although semiparametric transport maps have non-vanishing model bias, their ISE decays much faster than nonparametric transport maps, effectively leading to more accurate estimates of the predictive distributions in the challenging small-$N$ regime. The differences are most pronounced in the tail of the distribution of $Y \mid X$.}}
    \label{fig:ConvergenceOfISE}
\end{figure}

\section{Application to short-term TC intensity forecasting}\label{sec:TCApplication}

\subsection{Data and initial distribution of forecast intensity errors}
Human forecasters at the National Hurricane Center (NHC) issue intensity and track forecasts of tropical cyclone (TC) intensities using multiple statistical, AI, and physics-based models. The NHC's Official Forecast (OFCL) typically has higher skill than any single model because human experts weight various models based on environmental conditions. Nevertheless, the task of digesting multiple inputs in real-time is challenging for human forecasters. In addition, both physics and AI models can struggle with predicting extreme and rare events, including ``rapid intensification'' (RI) and ``rapid weakening'' events, which are {operationally defined as phenomena where a storm's maximum sustained winds increase or decrease by at least 30 knots (35 mph) within a 24-hour period. These rapidly changing storms pose massive threats to coastal communities and challenge up-to-date tracking and warnings during active hurricane seasons. An improvement in real-time 24-hour guidance with diagnostics that are easily interpretable to human forecasters would have a direct positive impact on emergency response and disaster preparedness.

The NHC keeps a historical database of their official forecasts and forecast errors over the years~\parencite{nhcDatabase}. The NHC ``forecast intensity error'' is defined as the difference between the official point forecast and the best track intensity at the forecast verifying time. Statistical uncertainty is typically taken as an average of these errors over several years of record. Following~\parencite{trabing2020}, we represent NHC intensity errors by a Gaussian distribution $N(\epsilon_{t+24}, \sigma^2)$, where the mean $\epsilon_{t+24}$ is the NHC forecast intensity error at t+24 hours, and the standard deviation $\sigma$ is a ten-year average of NHC forecast intensity errors at 24 hours. We refer to this distribution as the NHC base error model.

Our goal is to assess and improve upon the base model using  Statistical Hurricane Intensity Prediction Scheme (SHIPS; \parencite{demaria1994}) values. SHIPS is a statistical-dynamical model used by the NHC to forecast 
TC intensity changes at 12 to 72 hours with predictors that represent environmental values, including humidity, wind shear, and distance to land. Our transport maps use a subset of 13 SHIPS predictors and the intensities at times $t-12$, $t-6$, and $t$ hours (14 total predictors) (See Appendix~\ref{appendix:TC_data} for details.) We denote the entire 12-hour input sequence of $3 \times 14$ values by  $\mathbf{S}_{\leq t}$ or just $\mathbf{S}$,  and the density of residual errors at $t+24$ hours by $f(\epsilon_{t+24}\mid \mathbf{S}_{\leq t})$ or just $f(\epsilon\mid \mathbf{S})$.

To learn the diagnostic transport maps, we use storms in the Atlantic basin from 2000-2015. We then evaluate the results on storms from 2016-2022. Table \ref{tab:DataSplits} lists available data (number of 36-hour sequences and number of unique TCs) for calibration and testing. Category 3-5 storms (hurricanes achieving a wind speed with more than 95 knots) constitute tail events. Because of rapid changes in tail probabilities, both RI and RW events are directly related to tail events.

\begin{longtable}[]{@{}lll@{}} 
\caption{\small{{\bf Calibration and test data in TC application.} (i) We train the diagnostic transport maps on 12-hour sequences of SHIPS predictors for Atlantic storms from 2000-2015. (ii) We assess the performance of $t$+24 hour intensity error distributions on storms from 2016-2022. The table lists the number of available sequences and the number of unique TCs. Hurricane categories follow the Saffir-Simpson scale with division into tropical storm (TS) and categories 1-5. Rapid intensification (RI) and rapid weakening (RW) events are storms that change by more than 30 knots within 24 hours.}}\label{tab:DataSplits}\tabularnewline
\toprule\noalign{}
 & \# Sequences & \# Unique TCs  \\
\midrule\noalign{}
\endfirsthead
\toprule\noalign{}
 & \# Sequences & \# Unique TCs \\
\midrule\noalign{}
\endhead
\bottomrule\noalign{}
\endlastfoot
(i) Calibration (2000-2015) & & \\
Total & 2395 & 171 \\
RI & 204 & 59\\
RW & 60 & 25\\
TS & 1159 & 167\\
Categories 1-2 & 519 & 90\\
Categories 3-5 & 271 & 36\\
\hline 
(ii) Testing (2016-2022) & & \\
Total & 1274 & 95\\
RI & 120 & 28\\
RW & 30 & 13\\
TS & 664 & 91\\
Categories 1-2 & 221 & 44\\
Categories 3-5 & 161 & 21\\
 \end{longtable}

\subsection{Results} 
We fit a semiparametric diagnostic transport map to the calibration data by positing the Kumaraswamy distribution (Equation~\ref{eq:kumaraswamy}) on $Z \mid x$, where $Z=\hat F(Y|X)$ is the conditional PIT variable, and then learning the parameter map $x \mapsto \theta(x)$ with a three-layer convolutional neural network that feeds into a four-layer fully connected neural network with ReLU activation, where the maximum number of units per layer is allowed to increase with the calibration data size $N$. To find the best value of $\theta(x)$, we minimize an inverse-probability-weighted (IPW) negative log-likelihood, where the IPW weights are computed by Gaussian kernel density estimation.\footnote{See scikit-learn documentation on sklearn.neighbors.KernelDensity for details.}

For comparison, we also implement the fully nonparametric LADaR approach outlined in \textcite{dey2025}.  For the latter probabilistic classification approach, we minimize the Brier score for a three-layer convolutional neural network feeding into a four-layer monotonic neural network with the number of units per layer varying based on calibration size $N$. On a 13th generation Intel core processor with 64-bit CPU, it takes 1.84 seconds to fit the semiparametric map; at inference, semiparametric recalibration takes 0.20 seconds. The corresponding numbers for the nonparametric map trained for the same number of epochs are 91.12 seconds for fitting and 22.47 seconds for recalibration.

\subsubsection{Diagnostic insights}
The initial predictive distribution might capture some TC evolutionary modes better than others. Our methodology provides insights that can be connected back to the underlying physical mechanisms. Figure~\ref{fig:LDSAndPP} shows example diagnostics for the semiparametric diagnostic transport map. Figure~\ref{fig:Sebastian} illustrates how a forecaster can deploy diagnostic transport maps in real-time and assess whether the recommended corrections are consistent with physical processes and the storm's evolution. Semiparametric diagnostic transport maps improve the most upon NHC error distributions for intense storms and rapid intensity change events, which both involve tail probabilities. For example, the diagnostic plots in Figure~\ref{fig:LDSAndPP} identifies Point A as an input sequence with an unreliable initial predictive distribution. A closer look (see Figure~\ref{fig:Sebastian}) shows that Point A corresponds to a 12-hour sequence of Tropical Storm Sebastian (2019) when the storm had high zonal wind and low humidity, leading to a biased and overdispersed predictive distribution. Figure~\ref{fig:Sebastian} shows visuals for the corresponding evolutionary mode.

\begin{figure}
    \centering
    \includegraphics[width=0.75\linewidth]{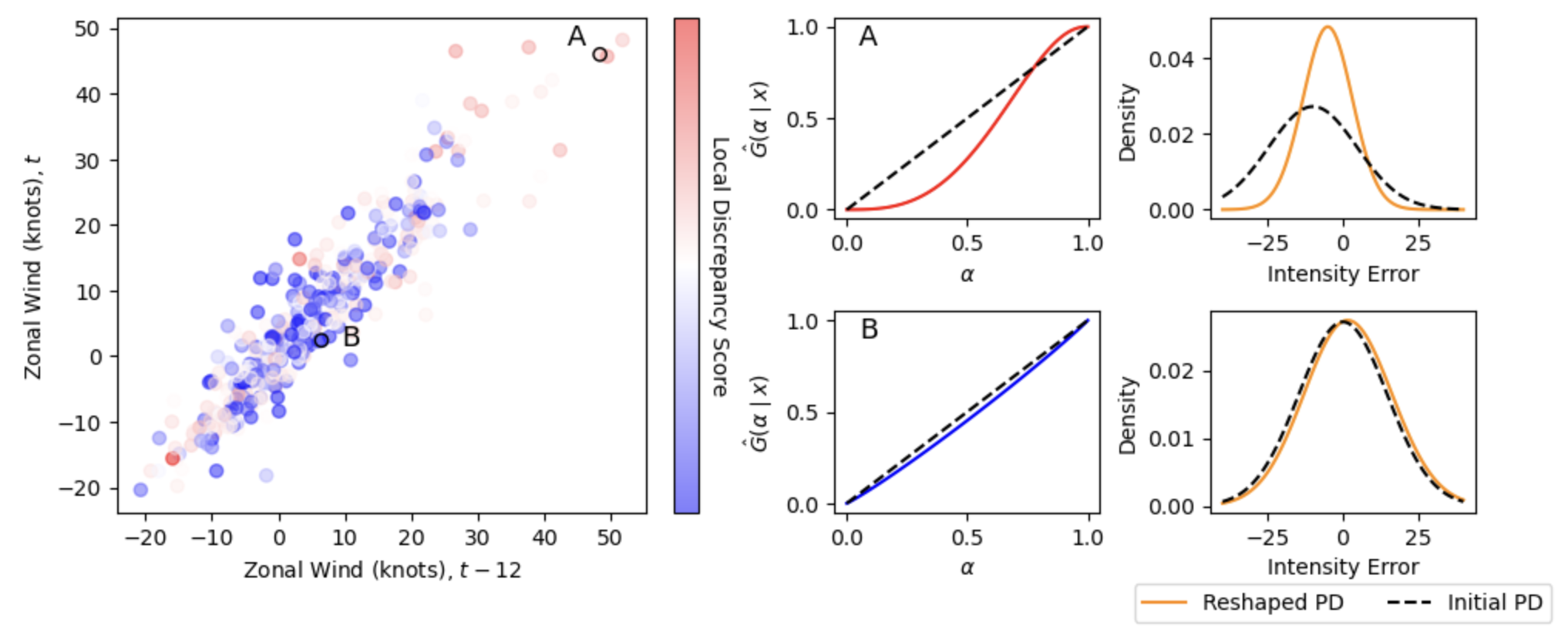}
    \caption{\small{{\bf Diagnostic insights for TC application.} {\em Left:} Each point in the graph represents a particular 12-hour input sequence $\mathbf{S}_{\leq t}=\{x_{t-12}, x_{t-6}, x_t\}$, where each $x$ corresponds to SHIPS predictors and TC intensity at a particular instance of the storm's evolution. With a semiparametric diagnostic transport map, we can quickly compute a local discrepancy score (LDS) that assesses how well the initial error distribution of $t+24h$ intensities matches the test data for that particular input sequence.  \textcolor{blue}{Point A} represents a storm sequence with low LDS (good fit).  \textcolor{red}{Point B} represents a storm sequence that has a high LDS (poor fit). For simplicity, we are  only displaying two SHIPS predictors here: the zonal wind at times $t-12$ and $t$ hours. {\em Center:} Diagnostic plots for test points A and B. Note that Panel A has a shape that indicates that the base model of +24-hour errors is negatively biased and overdispersed; that is,  the predictive distribution of TC intensities at t+24 hours is shifted downwards relative to the true TC intensity distribution with too high a variance.  {\em Right:} Recalibrated PDFs (\textcolor{orange}{orange}) of +24-hour intensity errors after applying the estimated transport maps to the base models.}}
    \label{fig:LDSAndPP}
\end{figure}

\subsubsection{Improved prediction and uncertainty quantification}
\begin{figure}
    \centering \includegraphics[width=0.75\linewidth]{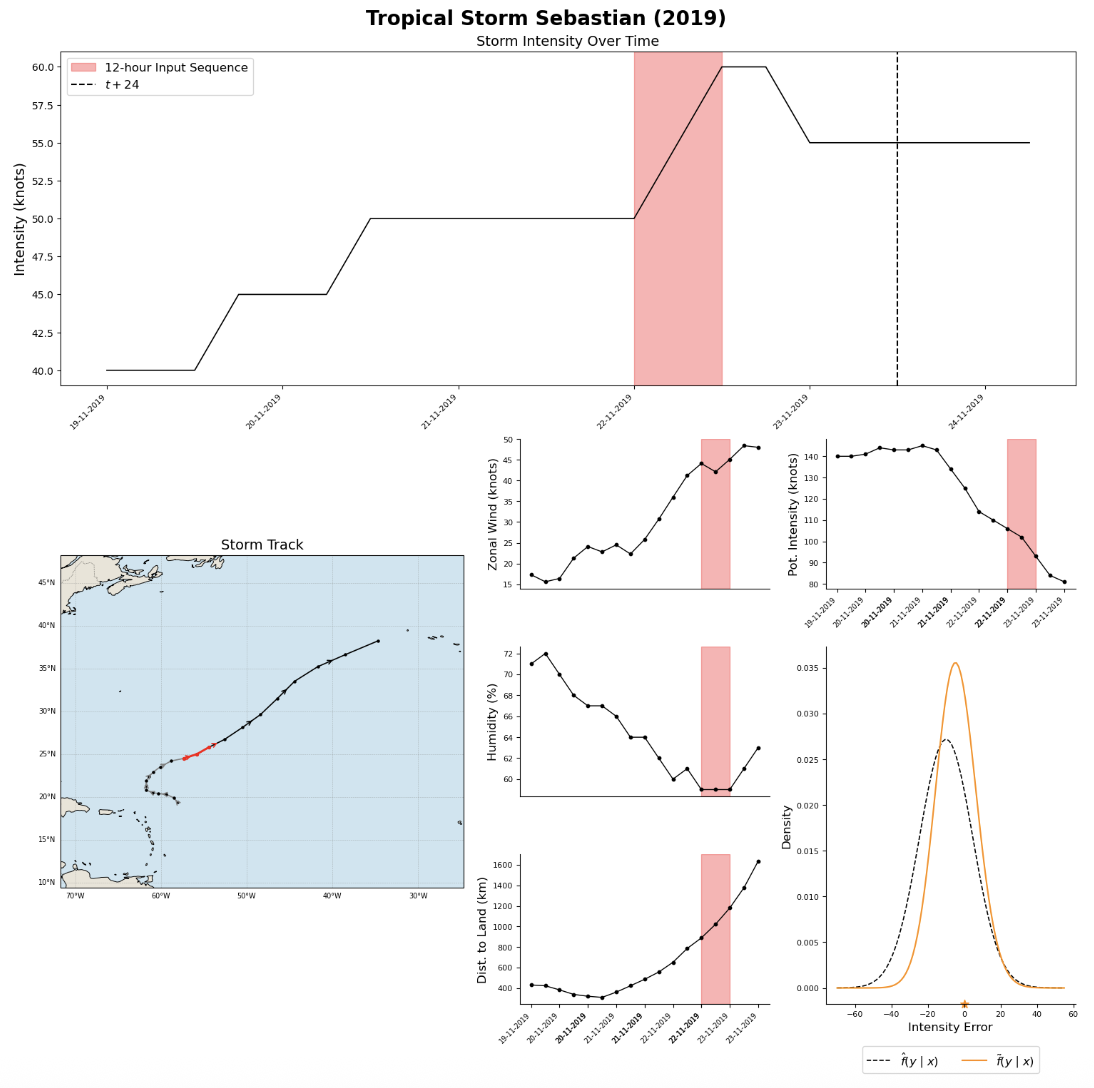}
    \caption{\small{{\bf Example dashboard for TC scientists and forecasters.} Diagnostic transport maps provide the user with local diagnostics and a mechanism for reshaping the base model in real time; no additional training is needed at deployment. Hence, a human expert can directly connect model output to physical processes and check whether the correction makes sense. This figure shows an example dashboard that can facilitate model verification for Tropical Storm Sebastian (2019). {\em Top:}  TC intensity as a function of time. The red band marks the time interval between $t-12h$, to $t$. The vertical dashed line marks the time $t+24h$ where we make our forecast. {\em Bottom left:} Track of Tropical Storm Sebastian (2019) with the current input sequence $t-12,...,t$ in red. {\em Bottom center:} 
    Four SHIPS predictors shown over time. Time points during which a rapid weakening event took place are marked with empty points. The red band marks the current input sequence. {\em Bottom right:} The error distribution $\hat f(\epsilon_{t+24h} \mid \mathbf{S}_{\leq t})$ from NHC official forecasts (black dashed) and the reshaped error distribution  $\widetilde f(\epsilon_{t+24h} \mid \mathbf{S}_{\leq t})$ after applying parametric transport maps  (\textcolor{orange}{solid orange}). Here the recalibration shifts the mean of the base density toward 0 knots (the true intensity error). For a video of our diagnostic transport maps applied to Tropical Storm Sebastian (2019) over time, see \href{https://drive.google.com/file/d/1asFsKGgfd6M58sY2WXtCOpRoflVhw7Sm/view?usp=drive_link}{here}.}}
    \label{fig:Sebastian}
\end{figure}

Next we examine whether semiparametric transport maps improve upon the NHC official forecasts, and compare the performance with fully nonparametric diagnostic transport maps. An analysis shows that reshaping the initial error distributions with SHIPS calibration data results in improved performance across a variety of metrics and subdivisions of the test data. Table~\ref{tab:RMSE} lists the root mean square  error (RMSE), a measure of the predictive accuracy of a point forecast derived from the error distribution, for both the semiparametric and nonparametric maps. Here we compute the RMSE of the expected value of the estimated error distribution in knots relative to the true error, where the latter is 0 by construction. Table~\ref{tab:CRPS} lists the continuous ranked probability score (CRPS; Section 1, Supplemental Materials), which is a commonly used measure of how well a full predictive distribution matches a realized outcome on average over a validation sample, for the semiparametric approach. We also list the threshold-weighted CRPS (twCRPS)~\parencite{allen2023} with the weight $\mathbb{I}\{y\geq y_{95}\}$ where $y_{95}$ is the 95th percentile of the intensity errors.

These results show that semiparametric diagnostic maps improve upon the NHC forecasts in CRPS, twCRPS, and RMSE when the values are averaged across the entire test set (``Overall'') as well as when they are computed for storm subdivisions based on hurricane strengths or intensity change; the fully nonparametric maps do not improve the CRPS or RMSE. The largest improvement is seen in the upper 5\% tail of the predictive distribution; for example, we have a 20\% decrease of the twCRPS relative the NHC base model for Categories 3-5 and RI events.

\noindent Our diagnostic tools enable scientists and forecasters to connect improvements in uncertainty quantification to real-time storm evolutions based on prior knowledge of NHC predictive performance under similar conditions. This allows scientists to not only track how their initial forecasts differ from calibration data in high-impact settings, but also understand why proposed changes to the PD are necessary based on environmental and evolutionary factors.} 

\begin{longtable}[]{@{}llll@{}}
\caption{\small{{\bf RMSE of NHC operational model before and after recalibration.} The RMSE of the TC intensity error distribution from the original base model (``NHC Operational forecast''), and after recalibration with diagnostic transport maps (``Nonparametric map'' and ``Semiparametric map'') with numbers in parenthesis denoting percentage improvement (that is, percentage decrease in the RMSE) over the NHC operational forecast. Semiparametric diagnostic transport maps improve the NHC error distribution across the board; the nonparametric transport maps do not improve.}}\label{tab:RMSE}\tabularnewline
\toprule\noalign{}
Category & NHC Operational forecast & Nonparametric map & Semiparametric map  \\
\midrule\noalign{}
\endfirsthead
\toprule\noalign{}
Category & NHC Operational forecast & Nonparametric map & Semiparametric map  \\
\midrule\noalign{}
\endhead
\bottomrule\noalign{}
\endlastfoot
Overall &  {12.11} &  {12.96 (-7.1\%)} &  {12.05 (+0.5\%)}\\
TS &  {9.72} &  {10.87 (-11.8\%)}  &  {9.72 (+0.0\%)}\\
Categories 1-2 &  {13.90} &  {15.04 (-8.2\%)} &  {13.69 (+1.5\%)}\\
Categories 3-5 &  {13.82} &  {14.63 (-5.9\%)} &  {13.66 (+1.2\%)}\\
RI &  {19.74} &  {19.89 (-0.8\%)} &  {19.26 (+2.4\%)}\\
RW &  {14.66} &  {14.85 (-1.3\%)} &  {14.54 (+0.8\%)}\\
 \end{longtable}

 \begin{longtable}[]{@{}llll@{}}
\caption{\small{{\bf CRPS/twCRPS of NHC operational model before and after recalibration.}  {The CRPS and a threshold-weighted CRPS with the weight $\mathbb{I}\{y\geq y_{95}\}$ where $y_{95}$ is the 95th percentile of the intensity errors before and after recalibration with the semiparametric diagnostic transport maps. The numbers in parenthesis denote the percentage improvement (that is, percentage decrease in the CRPS or twCRPS) over the NHC operational forecast. The largest improvement is seen in the 5\% upper tail of the predictive distribution; for example, we have a 20\% decrease of the twCRPS relative the NHC base model for Categories 3-5 and RI events.}}}\label{tab:CRPS}\tabularnewline
\toprule\noalign{}
\endfirsthead
\toprule\noalign{}
Category & NHC CRPS/twCRPS & Semiparametric CRPS & Semiparametric twCRPS\\
\midrule\noalign{}
\endhead
\bottomrule\noalign{}
\endlastfoot
Category & NHC CRPS/twCRPS & Semiparametric CRPS & Semiparametric twCRPS\\
\hline 
Overall &  {6.72} /  {0.28} &  {6.65 (+1.0\%)} &  {0.23 (+17.9\%)} \\
TS &  {5.67} /  {0.11} &  {5.66 (+0.2\%)} &  {0.09 (+18.2\%)} \\
Categories 1-2 &  {7.69} /  {0.48} &  {7.62 (+0.9\%)} &  {0.42 (+12.5\%)} \\
Categories 3-5 &  {7.72} /  {0.52} &  {7.58 (+1.8\%)} &  {0.43 (+20.9\%)} \\
RI &  {11.20} /  {0.20} &  {11.08 (+1.1\%)} &  {0.16 (+20.0\%)}\\
RW &  {8.30} /  {0.85} &  {8.17 (+1.6\%)} &  {0.75 (+11.8\%)} \\
 \end{longtable}
 
\section{Discussion} 
In this paper, we show how diagnostic transport maps can identify input patterns or evolutionary modes associated with poor calibration of predictive distributions, and how these maps then can morph the initial model toward locally calibrated distributions. In particular, we demonstrate the effectiveness of a semiparametric version for settings where calibration data are limited and accurate modeling of tail events is consequential. Here, ``semiparametric'' refers to imposing a low-dimensional parametric family on the probability-to-probability correction that is applied to the initial predictive distribution of $Y\mid X = x$, while still allowing the parameters of this correction to vary flexibly with $x$. For tropical cyclone intensity forecasting, semiparametric diagnostic transport maps improve operational NHC error distributions both on average and across a spectrum of extreme events. 

Although our diagnostic transport framework was originally motivated by severe weather forecasting, the proposed framework can be used in any application where scientists and forecasters need easy-to-use tools to assess and recalibrate black-box forecasting models with respect to target data of interest. It should be noted that, by construction, we only assess the base model {\em with respect to the predictors in the calibration set}. This means that a transformation can, in principle, degrade the original base model if key predictors are missing in the calibration data. By the same argument, diagnostic transport maps can also be used to assess the added value of censored or missing data that were not utilized in the original base model but are made available at the time of calibration.

\noindent As highlighted in this work, semiparametric maps are most useful in {\em small-sample} settings and when the base model is smooth, such as the commonly assumed standard normal error model. We show that under regular-variation tail assumptions on the transport map, our semiparametric estimator can target much smaller tail probabilities than the nonparametric smoother; nonparametric recalibration requires the target tail probability to exceed a polynomial threshold, whereas our semiparametric estimator can target probabilities that go to zero at a much faster exponential rate, but with potential model misspecification bias. Hence, nonparametric maps become attractive when larger calibration samples are available, allowing greater flexibility in reshaping the base model.  We recommend fully nonparametric transport maps if the base model is nonparametric and, for example, computed by a generative AI model. Future work on local diagnostics and reshaping of predictive distributions involves extending diagnostic transport maps to multivariate response variables and spatio-temporal processes, as well as quantifying the added value of incorporating additional data products as inputs, such as satellite imagery for the TC intensity forecasting problem.

\if1\anon
{
\section{Acknowledgements}
The authors would like to thank the STAtistical Methods for the Physical Sciences (STAMPS) Research Center at Carnegie Mellon University for support. ABL is grateful to Biprateep Dey, Mikael Kuusela, Jing Lei, and Larry Wasserman for valuable discussions on the original LADaR framework. The authors would also like to thank Dr. Kimberly Wood from the University of Arizona Department of Hydrology and Atmospheric Sciences and Dr. Tria McNeely at Microsoft for their helpful feedback and guidance on the TC application. Part of this work was funded by NSF DMS-2053804. RI is grateful for the financial support of CNPq (305065/2023-8 and 403458/2025-0) and FAPESP (grant 2023/07068-1 and 2025/24620-5).
}
\fi 

\section{Data availability statement}
The data that support the findings of this study are openly available at the following URLs: \href{https://rammb2.cira.colostate.edu/research/tropical-cyclones/ships/development_data/}{SHIPS - Developmental Data}, \href{https://www.nhc.noaa.gov/verification/verify7.shtml}{NHC Forecast Verification Data}, \href{https://www.aoml.noaa.gov/hrd/hurdat/Data_Storm.html}{HURDAT 2 Reanalysis Data}.

\section{AI disclosure statement}
During the preparation of this work, the authors used Anthropic’s Claude Sonnet 4.6 to summarize and edit content. After using this tool, the authors reviewed and edited the content as needed and take full responsibility for the contents of the publication.

\printbibliography

\begin{appendix}

\section{Notation}\label{sec:notation}
\begin{longtable}[]{@{}ll@{}}
\caption{\bf{List of symbols}}\label{tab:terminology}\tabularnewline
\toprule\noalign{}
Symbol & Meaning  \\
\midrule\noalign{}
\endfirsthead
\toprule\noalign{}
Symbol & Meaning  \\
\midrule\noalign{}
\endhead
\bottomrule\noalign{}
\endlastfoot

$X$ & covariates \\
$Y$ & response variable \\
$F$& CDF of the true predictive distribution of $Y\mid X$ \\
$\hat F$ & CDF of the initial base model for the distribution of $Y \mid X$ \\
$Z$ & local PIT variable, $Z :=\hat{F}(Y\mid X)$\\
$G$ & PIT-CDF, $G(\alpha \mid x) := \mathbb{P}(Z \leq \alpha \mid x)$ \\
$\hat{G}$& estimated PIT-CDF\\
$\theta$ & parameter for the parametric family imposed on $Z \mid X$\\
$G_{\theta}$ & PIT-CDF parameterized by $\theta$, $G(\alpha \mid x) \approx G_{\theta(x)}(\alpha)$\\
$\hat\theta$ & estimated parameter for $G_\theta$\\
$\tilde{F}$& reshaped CDF, $\tilde{F}(y \mid x) := \hat G \left( \hat F(y \mid x) \mid x \right)$\\
\end{longtable}

\section{Local diagnostics via PIT-CDF}\label{appendix:diagnostics}

Here we review the local diagnostics framework for conditional density models that first occured in \textcite{zhao2021}.
\subsection{Local calibration} Given a fixed baseline predictive distribution $\hat{F}_{Y\mid X}$, we define the conditional
PIT-CDF by
\begin{equation} 
G(\alpha \mid x) 
:= \mathbb{P}\left( \widehat{F}(Y|X) \le \alpha \mid x  \right),
\end{equation}
for $\alpha \in [0,1]$. The base model is {\em locally calibrated} at a fixed covariate value $x \in \mathcal{X}$
---that is, the predictive PDF $\hat f(y|x)=f(y|x)$ for every $y \in \mathcal{Y}$--- 
if, and only if, $\hat{F}(Y|X) \mid X=x$ is uniformly distributed on $[0,1]$, or equivalently
\[
G(\alpha \mid x) = \alpha \text{ for every }\alpha \in [0,1].
\]

\subsection{Diagnostic plots} To visually inspect whether our base model $\hat F_{Y\mid X}$ is locally calibrated at a given point $x \in \mathcal{X}$, we first evaluate our estimated PIT-CDF regression model $\hat G$ for an IID test sample $\{X_i, Y_i \}_{i=1}^n$, where each $X_i$ is drawn from some distribution over $\mathcal{X}$, and $Y_i|X_i$ is drawn from the true predictive distribution $F_{Y|X}$.\footnote{Calibration and test data typically represent independent samples from the same data-generating process $F_{Y\mid X}$. Both data sets are assumed to be independent from the data used to learn the base model $\hat F_{Y\mid X}$.} We then plot $\hat G(\alpha|x)$ versus $\alpha$, for a grid of $\alpha$-values in $[0,1]$. Figure~\ref{fig:cheatsheet} gives an example of how these ``diagnostic probability-probability plots'' can display failure models like bias, dispersion, skewness, and tailweight of the base model $\hat{F}_{Y|X}$, relative to the true predictive distribution $F_{Y|X}$ of the test data.\\

\begin{figure}
    \centering
    \includegraphics[width=0.8\linewidth]{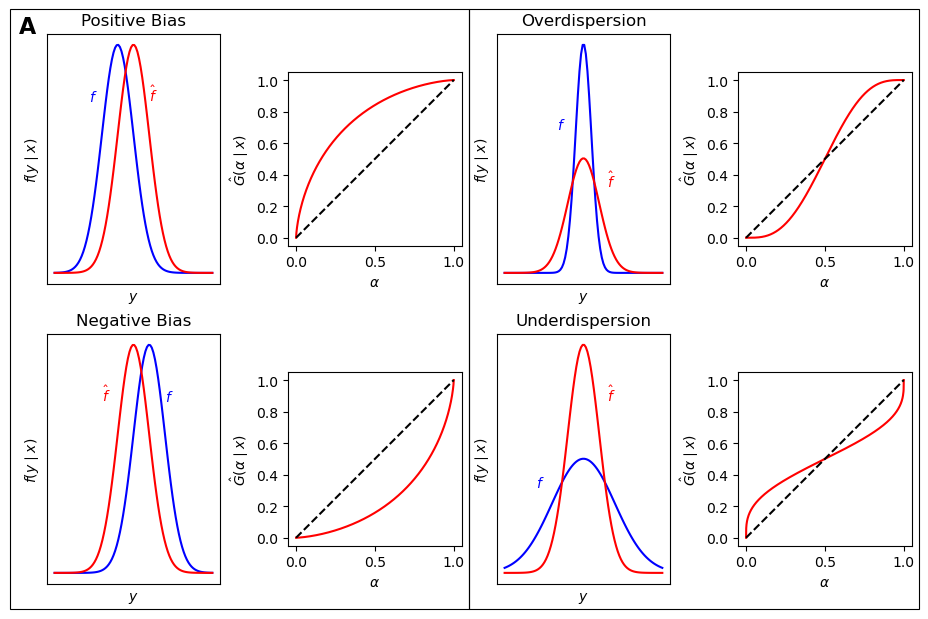}
    \includegraphics[width = 0.7\linewidth]{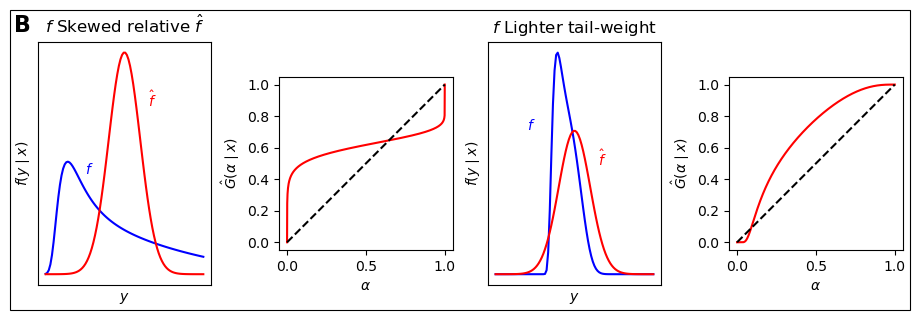}
    \caption{\small{{\bf{Recalibration of different failure modes in the PD.}} Examples showcasing how different failure modes of the initial density model $\hat{f}$ are diagnosed using probability-probability plots that map $\alpha \to \hat{G}(\alpha\mid x)$, where $\hat{G}(\cdot\mid x)$ is the estimated conditional PIT-CDF. {\em Panel A:} To the left, the initial density $\hat f$ is shifted toward larger outcome values (``Positive bias'') or toward smaller outcome values (``Negative bias'') relative to the true density $f$. To the right,  the initial density $\hat f$ is too wide (``Overdispersion'') or too narrow (``Underdispersion'') relative to the true density $f$.   {\em Panel B:} In these examples, $\hat f$ and $f$ have the same mean and variance. To the left, the true unknown density $f$ (\textcolor{blue}{blue}) is skewed to the right; hence a diagnostic plot that signals both positively bias and underdispersion.  To the right, the true unknown density $f$ (\textcolor{blue}{blue}) is also skewed to the right but with lighter tail-weight; hence a diagnostic plot that signals positive bias with slight overdispersion.}}
    \label{fig:cheatsheet}
\end{figure}

\subsection{Local discrepancy score} We define a {\em one-number summary statistic} to assess how well-calibrated the base model $\widehat{F}$ is, for a fixed $x \in \mathcal{X}$, according to our estimated PIT-CDF. More specifically, we compute the  {\em local discrepancy score} (LDS) according to 
\begin{equation}\label{eq:LDS}
    \text{LDS}(x) = \frac{1}{|\Gamma|}\sum_{\alpha \in \Gamma} \left(\hat{G}( \alpha \mid x)- \alpha \right)^2 ,
\end{equation}
where $\Gamma$ represents a set of points on an evenly spaced grid over $[0,1]$.  Geometrically, the LDS corresponds to the sum-of-squared deviation between the estimated PIT-CDF curve and the bisector in the local diagnostic plots. Note that when the PIT-CDF is used as a ``transport map'' to reshape the base model, we can directly relate the error in 
the reshaped predictive distribution \(\widetilde F\) at a point
\((x,y)\) to the discrepancy between the true transport map \(G\) and its parametric approximation at the probability \(p := \hat F(y \mid x) \in [0,1]\); see Corollary~\ref{cor:thresholded-forecasts} for details.

\section{Proofs}\label{app:proofs}

\begin{proof}[Proof of Lemma~\ref{lem:oracle-intro}]
Since $\hat{F}(\cdot\mid x)$ is strictly increasing, the events $\{Y\leq y\}$ and $\{\hat{F}(Y\mid x)\leq \hat{F}(y\mid x)\}$ coincide. Thus, $F(y\mid x) = \mathbb{P}(Y\leq y\mid X = x|x) = \mathbb{P}(\hat{F}(Y\mid x) \leq \hat{F}(y\mid x)|x) = G(\hat{F}(y\mid x)\mid x)$
\end{proof}

\begin{proof}[Proof of Corollary~\ref{cor:thresholded-forecasts}]
By Lemma~\ref{lem:oracle-intro}, \[F(y\mid x) = G(p\mid x), \quad \tilde{F}(y\mid x) = G_{\hat\theta}(p\mid x),\] so that the pointwise calibration error is \[|F(y\mid x) - \tilde F(y\mid x)| = |G(p\mid x) - G_{\hat\theta}(p\mid x)|.\] The conclusion follows from the triangle inequality.
\end{proof}

\begin{proof}[Proof of Proposition \ref{prop:stats-decomp}]
Write $\alpha=\hat F(y\mid x)\in[0,1]$.  By
Lemma~\ref{lem:oracle-intro}, $F(y\mid x)=G(\alpha\mid x)$, and by
construction $\tilde F(y\mid x)=G_{\hat\theta(x)}(\alpha)$.  Inserting the
intermediate term $G_{\theta^\star(x)}(\alpha)$ and applying the triangle
inequality gives the first display.  Taking the supremum over $y$ yields the second.
\end{proof}

\begin{proof}[Proof of Proposition~\ref{prop:stats-rate}]
By Proposition~\ref{prop:stats-decomp},
\[
\sup_y|\tilde F(y\mid x)-F(y\mid x)|\le A_\infty(x)+E_\infty(x).
\]
By Assumption~\ref{ass:lipschitz-stats},
\[
E_\infty(x)
\le L\|\hat\theta(x)-\theta^\star(x)\|.
\]
Assumption~\ref{ass:theta-rate-stats} gives
\[
E_\infty(x)=O_P\{r_N(x)\}.
\]
\end{proof}

\begin{proof}[Proof of Proposition \ref{prop:sp-tail-rate}]
The triangle inequality gives
$$
|\hat p_{\rm sp}(x)-p_N(x)|
\le
A_N(x)
+
\left|
G_{\hat\theta(x)}(\alpha_N)
-
G_{\theta^\star(x)}(\alpha_N)
\right|.
$$
By Assumption~\ref{ass:theta-rate-stats},
$$
\|\hat\theta(x)-\theta^\star(x)\|=O_P\{r_N(x)\}.
$$
Since $S(\alpha_N,x)r_N(x)\to0$,
$$
S(\alpha_N,x)\|\hat\theta(x)-\theta^\star(x)\|=o_P(1),
$$
and therefore, with probability tending to one,
$$
\|\hat\theta(x)-\theta^\star(x)\|\le \frac{\rho}{S(\alpha_N,x)}.
$$
On this event, Assumption~\ref{ass:relative-tail-stability} gives
$$
\left|
G_{\hat\theta(x)}(\alpha_N)
-
G_{\theta^\star(x)}(\alpha_N)
\right|
\le
C\,S(\alpha_N,x)\,
G_{\theta^\star(x)}(\alpha_N)\,
\|\hat\theta(x)-\theta^\star(x)\|.
$$
Dividing by $p_N(x)$ and using the rate for $\hat\theta(x)$ yields
$$
\frac{|\hat p_{\rm sp}(x)-p_N(x)|}{p_N(x)}
\le
\frac{A_N(x)}{p_N(x)}
+
O_P\!\left\{
S(\alpha_N,x)r_N(x)
\frac{G_{\theta^\star(x)}(\alpha_N)}{p_N(x)}
\right\}.
$$
Finally,
$$
G_{\theta^\star(x)}(\alpha_N)
\le
G(\alpha_N\mid x)
+
\left|
G_{\theta^\star(x)}(\alpha_N)-G(\alpha_N\mid x)
\right|
=
p_N(x)+A_N(x),
$$
so
$$
\frac{G_{\theta^\star(x)}(\alpha_N)}{p_N(x)}
\le
1+\frac{A_N(x)}{p_N(x)}.
$$
Substitution gives the result.
\end{proof}

\begin{lemma}\label{lem:true_conditional_pit}
    Fix $x \in \mathcal{X}$ and assume that the initial predictive CDF $\hat{F}(\cdot \mid x)$ is strictly increasing and continuous in $y \in \mathcal{Y}$, with $\hat{F}$ strictly bounded between [0,1]. For every $\alpha \in [0, 1]$
    \[G(\alpha \mid x) = F(\hat{F}^{-1}(\alpha\mid x)\mid x) .\]
Finally, $F(y \mid x) = \hat F(y \mid x)$ for all $y \in \mathcal{Y}$, if and only if, $G(\alpha \mid x)=\alpha$ for all $\alpha \in [0,1]$.
\end{lemma}

\begin{proof}
   By Definition~\ref{eq:cond-pit-cdf},
    \[G(\alpha\mid X = x) = \mathbb{P} \left(\hat{F}(Y \mid X)\leq \alpha \mid X = x \right) .\]
     Since $\hat{F}$ is a continuous, monotonically increasing function, we can compute its inverse $\hat F^{-1}$. Hence,
    \begin{align*}  G(\alpha\mid X = x) &= \mathbb{P} \left(Y \leq \hat F^{-1}(\alpha \mid X) \mid X=x \right) \\
   &=  F(\hat{F}^{-1}(\alpha\mid X)\mid X=x)
    \end{align*}
\end{proof}
The final result follows from the fact that the CDF fully characterizes the distribution of a random variable.

\section{Details on the TC data}\label{appendix:TC_data}

For calibration, we use the ``NHC forecast error''~\parencite{nhcDatabase}, which is defined as the residual of the NHC's operational intensity estimate relative to the best-track intensity data provided by NOAA's Hurricane Database 2 (HURDAT2)~\parencite{pastracks}. The 2022 SHIPS predictor file~\parencite{shipsPredictorFile} contains approximately 90 environmental predictors for the forecasting of hurricane intensity. The input sequence $\mathbf{S}_{\leq t}$ used to forecast the TC intensity error distribution $f(\epsilon_{t+24} \mid \mathbf{S}_{\leq t})$ at a 24-hour lead time include the TC intensity (in knots) and a subset of the following 13 SHIPS predictors at $t-12$, $t-6$ and $t$ hours: 
\begin{itemize}
    \item Maximum surface wind (knots)
    \item Shear magnitude (knots * 10) from 850-200 hPa
    \item Shear magnitude (knots * 10) from 850-200 hPa at 200-800 km
    \item Shear magnitude (knots * 10) with vortex removed and averaged from 0-500 km relative to the 850 hPa vortex center
    \item Relative humidity (\%) from 850-700 hPa at 200-800 km
    \item Relative humidity (\%) from 700-500 hPa at 200-800 km
    \item Relative humidity (\%) from 500-300 hPa at 200-800 km
    \item Maximum potential intensity from the Kerry Emanuel equation (knots)
    \item Zonal wind (knots * 10) at 200 hPa for 200-800 km
    \item Intensity change (knots) from $t=0$
    \item Distance to nearest major land mass (km)
    \item Latitude (degrees North * 10)
    \item Longitude (degrees West * 10)
\end{itemize}

\noindent SHIPS predictors are matched to intensity errors from the NHC's operational forecast through dates, times, latitudes, and longitudes. If a particular SHIPS predictor is null (represented by a 9999 in the file) for $t-12, t-6, $ and $t$, or if the NHC failed to report an intensity error at $t+24$, that storm sequence is removed entirely from the data set. If a subset of the data is missing, missing values were interpolated using the previous time point's measurement. Storm sequences are only considered once they are classified as a tropical cyclone (intensity $>34$ knots); tropical depressions are not considered.

\end{appendix}
\end{document}


\def\spacingset#1{\renewcommand{\baselinestretch}%
{#1}\small\normalsize} \spacingset{1}


\if1\anon
{
  \title{\bf Supplemental Material for Trustworthy Predictive Distributions for Tail Events with Semiparametric Diagnostic Transport Maps}
  \author{Elizabeth Cucuzzella\hspace{.2cm}\\
    Department of Statistics and Data Science, Carnegie Mellon University\\
    and \\
    Rafael Izbicki \\
    Department of Statistics, Federal University of Sao Carlos \\
    and \\
    Ann B. Lee\\
    Department of Statistics and Data Science, Carnegie Mellon University}
  \maketitle
} \fi

\if0\anon
{
  \bigskip
  \bigskip
  \bigskip
  \begin{center}
    {\LARGE\bf Title}
\end{center}
  \medskip
} \fi

\bigskip
\doublespacing

\section{CRPS in conditional density estimation (CDE)}

\noindent In the context of conditional density estimation, the continuous ranked probability score (CRPS) $L_{\text{CRPS}}$  measures the accuracy of an estimated conditional density $\tilde{f}(y\mid x)$ through its predictive cumulative distribution function (CDF) $\tilde{F}(y\mid x)$ according to
\begin{definition}[CRPS] The continuous ranked probability score is defined by
    \begin{equation} \label{eq:CRPS_CDE}
  L_{\text{CRPS}} (\tilde{f}; x, y) = \int_{-\infty}^{\infty} \big(\tilde{F}(t\mid x)  - \mathbb{I}(y\leq t)\big)^2dt,
\end{equation}
where $\tilde{F}(t\mid x)= \int_{-\infty}^t \tilde{f}(s\mid x) ds$. \label{def:CRPS}
\end{definition}

\noindent Given a validation sample $\mathcal{V}=\{(X_1,Y_1), \ldots, (X_m,Y_m)\}$, we estimate the CRPS loss on an evenly spaced grid $[t_1,...,t_m]$ according to 
\[\hat L_{\text{CRPS}}\tilde{f}; \mathcal{V}) =  \frac{\Delta_t}{m} \sum_{j=1}^{m} \left(\tilde{F}(t_j\mid X_i)- \mathbb{I}(Y_i\leq t_j)\right)^2\]
where $\Delta_t$ is the grid spacing between $t_{j}-t_{j-1}$.

\begin{proposition} (Local CRPS) \label{prop:CRPS_ISE}
    The conditional expectation of the CRPS loss given $X = x$ for fixed $\widetilde f$ is the integrated squared error (ISE) of the predictive CDF up to a constant that does not depend on $\tilde{F}$:
\begin{equation}
\mathbb{E} \big[ L_{\text{CRPS}}(\tilde{f}; X, Y)\mid x\big] =\int_{-\infty}^{\infty}\big(\tilde{F}(t\mid x)-F(t\mid x)\big)^2dt + K,
\end{equation}
where $K=\int^{\infty}_{-\infty}\mathbb{E}[(F(t\mid x) - \mathbb{I}(Y \leq t))^2\mid x]dt$.
\end{proposition}    

\begin{proof}
  The conditional expectation of the CRPS loss given $X = x$ is 
\[\mathbb{E}\big[ L_{\text{CRPS}}(\tilde{f}; X, Y)\mid x\big] = \mathbb{E}\big[\int_{-\infty}^{\infty} \big(\tilde{F}(t\mid x) - F(t\mid x) + F(t\mid x) - \mathbb{I}(Y\leq t)\big)^2dt\mid x\big]\]
\\
By expanding the square and by changing the order of expectation and integration, we have:
\begin{align*}
\mathbb{E}\big[L_{\text{CRPS}}(\tilde{f}; X, Y)\mid x\big] &= \mathbb{E} \big[ \int_{-\infty}^{\infty}\big(\tilde{F}(t\mid  x)- F(t\mid x)\big)^2dt\mid x\big] \\
 & + 2\int_{-\infty}^{\infty}\big(\tilde{F}(t\mid x)-F(t\mid x)\big)\mathbb{E}[(F(t\mid x) - \mathbb{I}(Y\leq t))\mid x] dt\\
 & + \int^{\infty}_{-\infty}\mathbb{E}[(F(t\mid x) - \mathbb{I}(Y \leq t))^2\mid x]dt
\end{align*}
Note that:
\begin{itemize}
    \item The first term represents the squared distance between $\tilde{F}$ and $F$ and is minimized when $\tilde{F}(\cdot \mid x) = F(\cdot\mid x)$.
    \item The second term equals zero
    \[\mathbb{E}[F(t\mid x)-\mathbb{I}(Y\leq t)\mid x]= F(t\mid x)-\mathbb{E}[\mathbb{I}(Y\leq t)\mid x] = F(t\mid x)-F(t\mid x)=0\]
    \item The third term is a constant that does not depend on $\tilde{F}$.
\end{itemize}
The result follows.
\end{proof}

\noindent For the synthetic example in Section 5, we  compute the ISE of the reshaped CDF at fixed $x$ via the conditional PIT-CDF and the base model CDF at $x$: 
 \begin{align}
    \mathcal{E}(x) &:=\int_{-\infty}^{\infty}\big(\tilde{F}(t\mid x)-F(t\mid x)\big)^2dt \nonumber  \\
    & = \int_{0}^{1} \frac{\big(\hat{G}(p\mid x)-G(p\mid x)\big)^2} {\hat{f} (u|x)} dp, \label{eq:true_ISE}
\end{align}    
where $u= \hat{F}^{-1}(p|x)$.\\

\section{Proper loss functions for conditional density estimation}\label{app:lossFunctions}

\begin{definition} (Proper loss function) Let $L (\tilde{f}; x, y)$ be a loss function  
$L: \mathcal{F} \times \mathcal{X} \times \mathcal{Y} \rightarrow \mathbb{R}$ describing the discrepancy between a conditional density estimator $\tilde{f}$ and the true conditional density $f$ of $Y$ given $X$. The loss $L (\tilde{f}; x, y)$ is proper if and only if
$$ f = \arg \min_{\tilde f \in \mathcal{F}} \ \mathbb{E} \! \left[ L (\tilde{f}; x, y) \right]. $$ 

\end{definition}
\begin{theorem}
    The CRPS loss function  $L_{\text{CRPS}} (\tilde{f}; x, y)$ is proper. 
\end{theorem}
\begin{proof}
 The result follows directly from Proposition~\ref{prop:CRPS_ISE}.
\end{proof}

\noindent Note that the standard formulation of the CRPS (Equation~\ref{eq:CRPS_CDE}) is equivalent to an integrated quantile (or pinball) loss function via 
\begin{align} 
  L_{\text{CRPS}} (\tilde{f}; x, y) &= \int_{-\infty}^{\infty} \big(\tilde{F}(t\mid x)  - \mathbb{I}(y\leq t)\big)^2dt \nonumber \\
  &= \int_{0}^{1} 2  L_{\alpha}(\tilde{F}^{-1};x,y) d\alpha
\end{align}
where $ L_{\alpha}(\tilde{F}^{-1};x,y) :=  \left(\tilde F^{-1}(\alpha \mid x)-y \right) \left(\mathbb{I}\left(y \leq \tilde{F}^{-1}(\alpha\mid x)\right) - \alpha \right)$ is the standard pinball loss of the quantile forecast $\tilde{F}^{-1}(\alpha|x)$  evaluated at $(x,y)$. This result can be found in~\textcite{laio2007,gneiting2011} with a general proof in Appendix A of~\textcite{fakoor2023}.\\ 

\noindent Depending on the application, one might choose threshold-weighted or quantile-weighted versions of the CRPS~\parencite{gneiting2011}
$$ L_{\text{CRPS}} (\tilde{f}; x, y) = \int_{-\infty}^{\infty} \big(\tilde{F}(t\mid x)  - \mathbb{I}(y\leq t)\big)^2 \ u(t) dt $$
and 
$$ L_{\text{CRPS}} (\tilde{f}; x, y) = \int_{0}^{1} 2  L_{\alpha}(\tilde{F}^{-1};x,y) \ v(\alpha) d\alpha, $$ with non-negative weight functions $u(\cdot)$ and $v(\cdot)$ on the real line $\mathbb{R}$ and unit interval $[0,1]$, respectively. \\

\noindent Finally, Theorem 1 in \textcite{zhao2021} shows that a uniform standard (marginal) PIT score does not necessarily imply that the conditional density is well estimated, whereas Theorem 2 shows that the {\em conditional} PIT-CDF uniquely characterizes the CDE.\\

\section{Optimal transport} \label{app:OT}
\noindent Here we briefly review Monge's continuous formulation of the OT distance; see \textcite{villani2021} for a more complete and mathematical description of OT.\\

\noindent Consider two distributions $F_0$ and $F$ defined over respective domains $\Omega_0, \Omega \subseteq \mathbb{R}^d$. Monge's OT problem is to find a function $T: \Omega_0 \rightarrow \Omega$ that pushes $F_0$ onto $F$ and minimizes the distance
$$ d(F_0, F) = \inf_T \int\| x - T(x) \|^p dF_0(x)$$
subject to $T(X) \sim F$, where $p \geq 1$ and $X \sim F_0$.
Such a minimizer $T^*$ exists if $F_0$ and $F$ both have densities. The minimizer $T^*$  is called the {\em optimal transport map}.  The map 
\begin{equation} \label{eq: transport_path}
    T_s(x) = (1-s) x + s T^*(x),   
\end{equation}
for $s \in [0,1]$
gives the path of a particle of mass at $x$. The  distribution of $T_s(X)$ is also the geodesic (that is, the shortest path according to the metric $d$) connecting $F_0$ to $F$; see \textcite{kolouri2017} for a technical definition.\\

\noindent If $\Omega=\Omega_0=\mathbb{R}$ and both $F$ and $F_0$ are absolutely continuous, with continuous densities which do not vanish, the optimal transport is given by
$$ T(x) = F^{-1} \left( F_0(x)\right).$$
See \textcite{santambrogio2015optimal} for a detailed survey.\\

\section{The diagnostic transport map corresponds to an optimal transport} \label{app:OT}
The PIT-based representation from  Lemma 3.1 admits a natural
interpretation in terms of one-dimensional optimal transport (OT). In this
view, the PIT-CDF \(G(\cdot\mid x)\) encodes a transport map between the
initial predictive distribution \(\hat F(\cdot\mid x)\) and the true
distribution \(F(\cdot\mid x)\) in probability space. An estimated diagnostic transport map
then induces a full transportation map in outcome space, providing a
continuous path from the original to the calibrated predictive distribution.

\subsection{One-dimensional OT between \texorpdfstring{\(\hat F(\cdot\mid x)\)}{Fhat} and \texorpdfstring{\(F(\cdot\mid x)\)}{F}}
Fix \(x \in \mathcal{X}\) and consider the two conditional distributions
$
\hat F(\cdot\mid x)$ and $
F(\cdot\mid x)$
of \(Y\) given \(X=x\). Let \(\mu_x\) and \(\nu_x\) denote the corresponding
probability measures on \(\mathbb{R}\). For a cost function
\(c(y,y') = |y - y'|^p\) with \(p \ge 1\), the one-dimensional Monge OT problem
seeks a measurable map \(T_x : \mathbb{R} \to \mathbb{R}\) that pushes
\(\mu_x\) forward to \(\nu_x\) and minimizes
$
\int |y - T_x(y)|^p \, d\mu_x(y).
$
In one dimension, when both distributions are continuous and strictly
increasing, the unique optimal map is the monotone rearrangement given by
\[
T_x(y) := F^{-1}\!\big(\hat F(y\mid x)\mid x\big).
\]
\subsection{Expressing the OT map in terms of the PIT-CDF}\label{OTmapfromPITCDF}
The OT map \(T_x\) can be expressed purely in terms of the base model CDF
\(\hat F\) and the PIT-CDF \(G\). Recall that for fixed \(x\),
\[
G(\alpha\mid x)
= \mathbb{P}\bigl(\hat F(Y\mid x) \le \alpha \mid X=x\bigr).
\]

\noindent In Section 3, we introduced Proposition 3.4, which we will now formally prove.
\begin{proof}
Fix \(x\) and define \(u := \hat F(y\mid x) \in [0,1]\). By the PIT identity, we may write
\[
G(\alpha\mid x) = F\bigl(\hat F^{-1}(\alpha\mid x)\mid x\bigr),
\qquad \alpha \in [0,1].
\]
For any \(\tau\in[0,1]\), the equation \(G(\alpha\mid x) = \tau\) is
equivalent to
\[
F\bigl(\hat F^{-1}(\alpha\mid x)\mid x\bigr) = \tau.
\]
Thus, if \(\alpha = G^{-1}(\tau\mid x)\), then
\[
F^{-1}(\tau\mid x)
= \hat F^{-1}(\alpha\mid x)
= \hat F^{-1}\!\bigl(G^{-1}(\tau\mid x)\mid x\bigr).
\]
Setting \(\tau = u = \hat F(y\mid x)\) yields
\[
T_x(y)
= F^{-1}\!\big(\hat F(y\mid x)\mid x\big)
= \hat F^{-1}\!\Big(G^{-1}\big(\hat F(y\mid x)\mid x\big)\,\Bigm|\,x\Big),
\]
as claimed.
\end{proof}

\noindent Proposition 3.4 shows that the OT map is fully determined by the
base model \(\hat F(\cdot\mid x)\) and the PIT-CDF \(G(\cdot\mid x)\).
In practice, we replace \(G\) by an estimate \(\widehat G\) obtained via
parametric or nonparametric models and invert both \(\widehat G\) and \(\hat F\)
numerically.\\

\printbibliography